\documentclass[11pt]{amsart}
\usepackage[
    doi=false,
    isbn=false,
    sortcites,
    backend=biber,  
    maxbibnames=99,
    ]{biblatex} 
\usepackage{amsmath,amssymb,amsthm,mathtools}
\usepackage{microtype}
\usepackage{fullpage}
\usepackage{xcolor}
\definecolor{josefpurple}{RGB}{128,0,128}

\usepackage{hyperref}
\hypersetup{
    hidelinks,
    pdftitle={High probability derivative bounds for random tanh neural networks on a hypercube},
    pdfauthor={Josef Dick, Michael Feischl, Fabian Zehetgruber},
    pdfkeywords={random neural networks, tanh activation, mixed derivatives, high-probability bounds, Xavier initialization, Gaussian concentration, quasi-Monte Carlo methods}
}
\usepackage{comment}

\title{High probability derivative bounds for random tanh neural networks on a hypercube}
\author{Josef Dick \and Michael Feischl \and Fabian Zehetgruber}

\newcommand{\R}{\mathbb{R}}
\newcommand{\C}{\mathbb{C}}
\newcommand{\N}{\mathbb{N}}
\newcommand{\E}{\mathbb{E}}
\newcommand{\Prob}{\mathbb{P}}
\newcommand{\cG}{\mathcal{G}}
\newcommand{\cT}{\mathcal{T}}

\newcommand{\cB}{\mathcal{B}}
\newcommand{\cR}{\mathcal{R}}

\newcommand{\cE}{\mathcal{E}}

\newcommand{\cW}{\mathcal{W}}

\newcommand{\cI}{\mathcal{I}}
\newcommand{\cQ}{\mathcal{Q}}
\newcommand{\cA}{\mathcal{A}}
\newcommand{\cJ}{\mathcal{J}}
\newcommand{\cK}{\mathcal{K}}

\newcommand{\sphere}{S}

\newcommand{\chara}{\mathbf{1}}

\newcommand{\ft}{\mathfrak{t}}
\newcommand{\fH}{\mathfrak{H}}

\newcommand{\fF}{\mathfrak{F}}

\DeclareMathOperator{\tangent}{tang}
\DeclareMathOperator{\mat}{mat}
\DeclareMathOperator{\coeff}{coeff}
\DeclareMathOperator{\im}{Im}
\DeclareMathOperator{\lexmin}{lexmin}

\DeclareMathOperator{\proj}{proj}

\newtheorem{theorem}{Theorem}
\newtheorem{lemma}[theorem]{Lemma}
\newtheorem{proposition}[theorem]{Proposition}
\newtheorem{corollary}[theorem]{Corollary}
\theoremstyle{definition}

\theoremstyle{remark}
\newtheorem{remark}[theorem]{Remark}

\begin{document}

\begin{abstract}
We establish high-probability bounds for mixed input derivatives of wide random neural networks whose activation derivatives satisfy a factorial growth bound. Our main result specializes these
estimates to $\tanh$ networks with Xavier initialization. A direct deterministic analysis based on Euclidean operator norms of the weight matrices yields derivative bounds that generally grow exponentially with the depth. We show that this growth can be substantially improved for sufficiently wide Gaussian networks by isolating the term that is linear in the highest-order derivative and controlling the corresponding tangent directions by measurable finite nets.

For scalar-output $\tanh$ networks with Gaussian weights and Xavier initialization, we prove that there exist constants $C,C_0,C_1>0$ such that, whenever the common hidden width satisfies $n \geq C\left(L^3n_0^2(1+\log n_0)+L^2\left(1+\log(L/\eta)\right)\right)$, then, with probability at least $1-\eta$, the estimate $\left|D^u\cR_{\Phi^{(L)}}(x)\right| \leq C_0 |u|! (C_1L)^{|u|-1}\prod_{j\in u}\beta_j(\eta,n_0)$ holds simultaneously for every non-empty $u\subseteq[n_0]$ and every $x\in[0,1]^{n_0}$. Thus, the first-order derivative bound is independent of the depth, while a square-free mixed derivative of order $|u|$ grows at most polynomially as $L^{|u|-1}$, apart from the coordinate factors. As consequences, we obtain high-probability bounds for the Euclidean Lipschitz constant and for weighted Sobolev norms of the network realization. The latter connect the derivative estimates to quasi-Monte Carlo integration and indicate how such regularity can enter the analysis of QMC-based training.
\end{abstract}

\maketitle

\section{Introduction}

Neural networks are empirically known to be vulnerable to adversarial
perturbations: small, deliberately chosen changes of the input can lead to
large changes in the output
\cite{szegedy_adversarial,goodfellow_adversarial}. A natural worst-case
measure of this sensitivity is the Euclidean Lipschitz constant. For a
scalar-valued map $f:D\subseteq\R^d\to\R$, we write
\[
    \operatorname{Lip}_{2}(f;D)
    :=
    \sup_{\substack{x,y\in D\\x\neq y}}
    \frac{|f(x)-f(y)|}{\|x-y\|_2}.
\]
If $D$ is convex and $f$ is continuously differentiable, then $ \operatorname{Lip}_{2}(f;D) \leq \sup_{x\in D}\|\nabla f(x)\|_2$. Thus, for neural networks with a smooth activation function such as
$\tanh$, robustness can be studied through bounds on input derivatives. In the worst case the Lipschitz constant behaves like the product of layerwise operator norms. For random neural networks we can hope for substantially smaller derivative bounds with high probability.

Derivative control is useful well beyond adversarial robustness. In scientific computing one often seeks a cheap surrogate for a high-dimensional data-to-observable map. Given $y\in Y\subseteq\R^s$, one would like to compute $G(y)\in\R^{N_{\rm obs}}$, where evaluating $G(y)$ requires solving a parameter-dependent partial differential equation. Such many-query problems arise, for example, in
uncertainty quantification, inverse problems, optimization, and design. When each PDE solve is expensive and the number $s$ of uncertain parameters is large, constructing an accurate surrogate can itself become a high-dimensional problem. Neural-network surrogates trained at low-discrepancy or quasi-Monte Carlo (QMC) points provide one approach to this problem \cite{mishra_rusch_low_discrepancy,longo_higher_order_qmc,kuo_regularity}. Their analysis naturally leads to mixed input derivatives: QMC error bounds and weighted Sobolev norms are governed not only by first derivatives, but by families of mixed derivatives and by their dependence on the active coordinate set. Consequently, high-probability mixed-derivative estimates for randomly initialized smooth networks can simultaneously inform robustness questions and the regularity theory underlying QMC-based approximation and training.

\subsection{Related work.}
Derivative bounds for neural networks arise in several distinct literatures. On the deterministic side, the connection between adversarial robustness and Lipschitz regularity has motivated methods for estimating or certifying Lipschitz constants of neural networks \cite{virmaux_scaman_lipschitz,jordan_dimakis_lipschitz,
fazlyab_lipschitz}, as well as direct certification of higher-order quantities such as Hessians \cite{sharifi_fazlyab_hessian}. For smooth
approximation, simultaneous approximation of functions and their
derivatives by feedforward networks goes back at least to
\cite{hornik_derivatives}, while quantitative approximation in high-order
Sobolev norms by $\tanh$ networks was studied in \cite{deryck_tanh}. In
scientific computing and QMC-based learning, mixed derivatives play a
central role in weighted Sobolev estimates and in the analysis of
generalization from low-discrepancy point sets
\cite{mishra_rusch_low_discrepancy,longo_higher_order_qmc,
kuo_regularity,keller_lattice_survey}. In particular, Keller, Kuo,
Nuyens, and Sloan derive explicit mixed-derivative bounds for smooth deep
networks under deterministic restrictions on the network parameters
\cite{kuo_regularity}. Our deterministic estimates are closest to this
last line of work, but are formulated in terms of Euclidean operator norms
in preparation for the probabilistic analysis.

A second line of work concerns derivatives of randomly initialized neural
networks, with much of the literature focusing on first-order quantities.
The singular values and stability of the input-output Jacobian at random
initialization have been studied through dynamical-isometry and random-matrix
methods
\cite{pennington_dynamical_isometry,hanin_exploding_gradients,
hanin_nica_products,pastur_slavin_jacobian}. These works reveal in
particular how depth, width, initialization, and the activation function
affect propagation of gradients through a random network. Closest to the
Lipschitz consequence of our results, high-probability upper and lower
bounds for Lipschitz constants of random ReLU networks were established in
\cite{geuchen_random_lipschitz}, and near-optimal estimates for
$\ell^p$-Lipschitz constants of deep random ReLU networks were subsequently
obtained in \cite{dirksen_lipschitz}.

A third neighboring theory concerns the large- and infinite-width behavior of randomly initialized neural networks. At fixed depth, wide fully connected networks with random weights converge to Gaussian processes \cite{lee_deep_gp,matthews_gp,hanin_infinite_width_gp}, and functional versions of this convergence describe the network as a random function, rather than only through its values at finitely many inputs \cite{bracale_functional_gp}. Quantitative finite-width Gaussian approximations were subsequently obtained in \cite{basteri_trevisan_gaussian,balasubramanian_gaussian_field}. Particularly close to the present work are the quantitative central limit theorems of Favaro, Hanin, Marinucci, Nourdin, and Peccati \cite{favaro_quantitative_clt}. They study deep fully connected networks with Gaussian weights at large but finite width and prove quantitative approximation of the network, together with its input derivatives, by the
corresponding infinite-width Gaussian process. Their results apply at any
fixed network depth and include functional approximations of the entire
random field, rather than only pointwise statements. In particular, for a
fixed derivative order and sufficiently smooth activation functions, their
results provide simultaneous control of the finite-width network and its
derivatives through their distance to a smooth limiting Gaussian field.
Thus, neither the consideration of higher input derivatives of random
smooth networks nor uniformity over a continuum of input points is by
itself specific to the present work. The nature of the estimate obtained here is, however, different. The
results of \cite{favaro_quantitative_clt} are Gaussian-approximation theorems: schematically, for fixed depth $L$ and fixed derivative order $k$, they show that $\cR_{\Phi_n}$ and its derivatives up to order $k$ are close, as random fields, to the corresponding derivatives of an infinite-width Gaussian process, with an error that tends to zero as the
width $n$ tends to infinity. Our objective is instead to bound the derivatives of the finite-width
network directly and nonasymptotically. 

Related finite-width information is obtained from a different perspective by Hanin \cite{hanin_perturbative}, who develops a perturbative $1/n$-expansion for joint cumulants of the outputs and derivatives of Gaussian fully connected networks. A notable conclusion of that analysis is that the ratio of depth to width acts as an effective expansion parameter, so that terms nominally of order $n^{-r}$ can acquire powers of
the depth. This provides complementary evidence that depth and width must
be analyzed jointly when studying derivatives of random deep networks.
Whereas \cite{hanin_perturbative} describes the distributional structure
of the network and its derivatives through cumulant expansions, our result
provides a direct samplewise high-probability regularity estimate.

At first order, Kim and Yang
\cite{kim_yang_jacobian} establish a joint infinite-width
Gaussian-process limit for a multilayer perceptron and its input Jacobian,
and use this limit to study Jacobian-regularized training. This complements
the random-matrix and dynamical-isometry literature on input-output
Jacobians
\cite{pennington_dynamical_isometry,hanin_exploding_gradients,
hanin_nica_products,pastur_slavin_jacobian} and the high-probability
Lipschitz estimates for random ReLU networks
\cite{geuchen_random_lipschitz,dirksen_lipschitz}.

To the best of our knowledge, the preceding results do not provide the
particular nonasymptotic regularity estimate proved here. More precisely,
we are not aware of a finite-width result that gives a single
high-probability event on which every non-empty square-free mixed input
derivative is bounded simultaneously and uniformly over the input domain,
with an explicit product-and-order-dependent bound and with explicit
dependence on the derivative order, depth, width, input dimension, and
failure probability. This explicit depth dependence is particularly
important: a direct deterministic argument based on products of
layerwise operator norms generally leads to exponential growth in $L$,
whereas our probabilistic argument yields the polynomial factor
$L^{|u|-1}$.

\subsection{Setting and main result.}
Throughout the paper we consider fully connected neural networks
\[
    \varphi^{(\ell)}=(W^{(0)},\ldots,W^{(\ell)}),
    \qquad \ell\in\N_0,
\]
where, since the biases are fixed rather than part of the parameter variable, the parameter space is
\[
    \Theta_\ell
    :=
    \prod_{r=0}^{\ell}
    \R^{n_{r+1}\times n_r}.
\]
The network has input dimension $n_0$ and layer widths $n_1,n_2,\ldots$. For
$x\in\R^{n_0}$, define
\begin{equation}\label{eq:zero_depth_realization}
    \cR_{\varphi^{(0)}}(x):=b^{(1)}+W^{(0)}x,
\end{equation}
and, recursively for $\ell\in\N$,
\[
    \cR_{\varphi^{(\ell)}_i}(x)
    =
    b_i^{(\ell+1)}
    +
    \sum_{j=1}^{n_\ell}W^{(\ell)}_{i,j}
    \sigma\!\left(\cR_{\varphi^{(\ell-1)}_j}(x)\right),
    \qquad i=1,\ldots,n_{\ell+1}.
\]
The activation $\sigma:\R\to\R$ is smooth, and the main probabilistic result
is stated for $\sigma=\tanh$. The biases are arbitrary but fixed; all
estimates below are uniform in their values. For a scalar-output network with $L\geq2$ hidden layers of common width $n$, we assume independent centered Gaussian weights with Xavier scaling
\cite{glorot_bengio}:
\[
    W^{(0)}_{i,j}\sim\mathcal N\!\left(0,\frac{2}{n_0+n}\right),
    \qquad
    W^{(\ell)}_{i,j}\sim\mathcal N\!\left(0,\frac1n\right)
    \quad (1\leq\ell\leq L-1),
\]
and
\[
    W^{(L)}_{1,j}\sim\mathcal N\!\left(0,\frac{2}{n+1}\right).
\]
For $j\in[n_0]$ and $\eta\in(0,1)$, set
\[
    \beta_j(\eta,n_0)
    :=
    \sqrt{\frac{2}{n_0+n}}
    \left(
        n+2\sqrt{n\log\!\left(\frac{4n_0}{\eta}\right)}
        +2\log\!\left(\frac{4n_0}{\eta}\right)
    \right)^{1/2}.
\]
Our main theorem shows that there exist universal constants
$C,C_0,C_1>0$ such that the width condition
\[
    n
    \geq
    C\left(
        L^3n_0^2(1+\log n_0)
        +
        L^2\left(1+\log\frac{L}{\eta}\right)
    \right)
\]
implies, with probability at least $1-\eta$, that
\begin{equation*}
    \left|D^u\cR_{\Phi^{(L)}}(x)\right|
    \leq
    C_0 |u|! (C_1L)^{|u|-1}
    \prod_{j\in u}\beta_j(\eta,n_0)
\end{equation*}
holds simultaneously for every non-empty $u\subseteq[n_0]$ and every
$x\in[0,1]^{n_0}$. In particular, the first-derivative bound is independent
of the depth. After increasing the universal constant in the width
condition if necessary, $\beta_j(\eta,n_0)\leq3$, and therefore the same
event gives the Euclidean robustness estimate $\operatorname{Lip}_{2}\!\left(\cR_{\Phi^{(L)}};[0,1]^{n_0} \right) \leq 3C_0\sqrt{n_0}$, again with no growth in $L$.

\subsection{Contributions.}
The proof has four main components.
\begin{enumerate}
    \item We first derive deterministic mixed-derivative estimates for
    smooth neural networks in terms of Euclidean operator norms of the
    weight matrices. This is the $2$-norm analogue of the regularity
    argument in \cite{kuo_regularity}. As expected, direct multiplication
    of layerwise matrix norms generally produces an exponential dependence
    on the depth.

    \item We then isolate, in the Fa\`a di Bruno expansion, the partition
    consisting of the full derivative set. This contribution is linear in
    the highest-order derivative and is governed by a tangent vector. Under
    a suitable tangent condition, a majorant-series argument allows us to replace the exponential depth factor by a polynomial dependence on $L$.

    \item For Gaussian random networks we prove that the required matrix
    and tangent conditions hold with high probability. The key step is to
    construct measurable finite nets for the normalized tangent directions
    and to exploit the independence of the current weight matrix from the
    preceding layers. This replaces a full operator-norm factor by a
    layerwise factor of order $1+L^{-1}$ along the relevant tangent
    directions and leads to the width condition stated above.

    \item Finally, we record consequences of the simultaneous derivative
    event for weighted Sobolev norms, QMC quadrature and lattice-based
    training, and Euclidean Lipschitz continuity. The QMC interpretation
    also identifies a limitation of isotropic Xavier initialization: the
    coordinate factors $\beta_j$ do not decay with $j$, so
    dimension-independent weighted QMC estimates require additional
    anisotropy or regularization.
\end{enumerate}

The paper is organized as follows. We first establish deterministic
derivative estimates and the depth-stable majorant bound under the tangent
condition. We then prove the high-probability matrix and tangent estimates
for wide Gaussian networks and specialize them to scalar-output Xavier
networks. The subsequent section derives the QMC and Lipschitz consequences.
Auxiliary analytic, combinatorial, probabilistic, and measurability results
are collected in the appendix.

\subsection{Notation}

We set $\N:=\{1,2,\ldots\}$ and $\N_0:=\N\cup\{0\}$. For $s\in\N$, set
$[s]:=\{1,\ldots,s\}$. By $e_j$ we denote the $j$th canonical basis vector,
with the ambient dimension clear from the context. The Euclidean norm of a
vector is denoted by $\|\cdot\|_2$, and the induced Euclidean operator norm
of a matrix by $\|\cdot\|_{2\to2}$. We write
$\sphere^{m-1}:=\{x\in\R^m:\|x\|_2=1\}$.

For $a,b\in\R^n$, the componentwise, or Hadamard, product is denoted by
$a\odot b$. More generally, for $a^{(i)}\in\R^n$, $i=1,\ldots,r$, we write
$\bigodot_{i=1}^r a^{(i)}$. For $\nu\in\N_0^s$, define
$|\nu|:=\sum_{i=1}^s\nu_i$ and
$\partial^\nu:=\prod_{j=1}^s\partial_j^{\nu_j}$. If $u\subseteq[s]$, we
write $D^u:=\prod_{j\in u}\partial_j$; in particular, the derivatives
$D^u$ considered in the main probabilistic result are square-free mixed
derivatives.

We use the conventions $\prod_{i=1}^0a_i = \prod_{v\in\emptyset}a(v) =1$ and $ \sum_{i=1}^0a_i = \sum_{v\in\emptyset}a(v) = 0$, as well as $0^0=1$ and $\sigma^{(0)}:=\sigma$. For a non-empty finite set $u$, let $\Pi(u)$ be the set of all partitions
of $u$, and set $\Pi(\emptyset):=\{\emptyset\}$. For
$r\in\{1,\ldots,|u|\}$, let $\Pi_r(u)$ be the set of partitions with
exactly $r$ blocks. For a polynomial $p$ and $k\in\N_0$, $\coeff_k(p)$
denotes the coefficient of $z^k$. All finite-dimensional matrix and parameter spaces are equipped with their
usual Euclidean topology and the corresponding Borel $\sigma$-algebra.
Products of parameter spaces are equipped with the product topology.

\section{Derivative bounds of neural networks}

Derivative bounds for smooth neural networks were obtained in \cite{kuo_regularity} using the $\infty$-matrix norm of the weight matrices. We follow the same basic argument, but formulate it in terms of the Euclidean operator norm. This is convenient for Gaussian random matrices and will be used in the probabilistic part of the paper. If $n_0, L \in \N$ and $\beta_j, \kappa_\ell > 0$ for every $j \in [n_0]$ and $\ell \in [L]$, then we define $\Theta^{\mat}_{L}(n_0,  \beta, \kappa)$ as the set of all $\varphi \in \Theta_L$ such that $\|W^{(0)}e_j\|_2 \le \beta_j$ holds for every $j \in [n_0]$ and $\|W^{(\ell)}\|_{2 \to 2}\le \kappa_\ell$ holds for every $\ell \in [L]$.

We denote by $\kappa = (\kappa_\ell)_{\ell \in [L]}$ the tuple of all $\kappa_\ell$. We introduce notation that allows us to write down the Fa\`a di Bruno formula in a way that is convenient for us. We define the function $\ft: \N_0^{n_0} \to \bigcup_{r \in \N_0} \N^r$ such that for all $\nu \in \N_0^{n_0}$ we define $\ft(\nu) =: a \in [n_0]^{|\nu|}$ as the unique vector that is nondecreasing and such that $\partial^\nu = \prod_{k = 1}^{|\nu|} \partial_{a_k}$. For any $v \subseteq [|\nu|]$ we define
\begin{align}\label{eq:def_of_mu_nu_v}
    (\mu(\nu, v))_j = \left|\left\{k \in v \mid (\ft(\nu))_k = j\right\}\right| = \left|\left(\ft(\nu)\right)^{-1}(j) \cap v\right| \quad \text{for all } j = 1, \ldots, n_0
\end{align}
and we note that $|\mu(\nu, v)| = \sum_{j = 1}^{n_0} (\mu(\nu, v))_j = |v|$ and $\sum_{v \in \pi}\mu(\nu, v) = \nu$. This notation will be useful later.

Throughout this paper we always assume that for every $r \in \N$ there exists $A_r > 0$ such that $\left|\sigma^{(r)}(x)\right| \leq A_r$ holds for all $x \in \R$. We define $B^{(0)}_1(\kappa) := 1$ and $B^{(0)}_q(\kappa) := 0$ for all $q > 1$. For all $\ell \in [L]$ and all $q \in \N$ we define
\begin{align*}
    B^{(\ell)}_q(\kappa) := \kappa_\ell \sum_{\pi \in \Pi([q])} A_{|\pi|} \prod_{v \in \pi} B^{(\ell - 1)}_{|v|}(\kappa).
\end{align*}
For all $\ell \in [L]$, all $r \in \N_0$, all $\varphi^{(\ell - 1)} \in \Theta_{\ell - 1}$ and all $x \in \R^{n_0}$ we define the vector $h^{(\ell)}_{r}(x, \varphi^{(\ell - 1)}) \in \R^{n_\ell}$ by
\begin{align*}
    h^{(\ell)}_{r, i}(x, \varphi^{(\ell - 1)}) := \sigma^{(r)}\left(\cR_{\varphi^{(\ell - 1)}_i}(x)\right) \quad \text{for all} \quad i \in \{1, \ldots, n_\ell\}.
\end{align*}

\begin{theorem}\label{thm:derivative_bound_from_matrix_norm_bound}
    If for every $r \in \N$ there exists $A_r > 0$ such that $\left|\sigma^{(r)}(x)\right| \leq A_r$ holds for all $x \in \R$, then  
    \begin{align*}
        \left\|\partial^\nu\cR_{\varphi^{(\ell)}}(x)\right\|_2 \leq B^{(\ell)}_{|\nu|}(\kappa) \prod_{j = 1}^{n_0} \beta_j^{\nu_j},
    \end{align*}
    holds for all $n_0 \in \N$, all $\ell\in\{0,\ldots,L\}$, all $\varphi^{(\ell)} \in \Theta^{\mat}_\ell(n_0, \beta, \kappa)$, all $\nu \in \N_0^{n_0} \setminus\{0\}$, all $x \in \R^{n_0}$.
\end{theorem}
\begin{proof}
    We prove the result by induction over the layer index $\ell$. Suppose $\ell = 0$ and suppose $\nu \in \N_0^{n_0}$ is such that $\nu_k=1$ for one $k\in[n_0]$ and $\nu_m=0$ for all $m\in[n_0]\setminus\{k\}$. Then
    \begin{align*}
        \left\|\partial^\nu\cR_{\varphi^{(0)}}(x)\right\|_2 = \left(\sum_{i = 1}^{n_1} \left|W^{(0)}_{i,k}\right|^2\right)^{1/2} = \left\|W^{(0)} e_k\right\|_2 \leq \beta_k = B^{(0)}_{|\nu|}(\kappa) \prod_{j = 1}^{n_0} \left(\beta_j\right)^{\nu_j}
    \end{align*}
    If $|\nu| > 1$ then the derivative is zero and the inequality is trivially true. For the induction step, assume that $\ell \in \N$ and that the induction hypothesis holds for $\ell - 1$. We have
    \begin{align*}
        \left\|\partial^\nu\cR_{\varphi^{(\ell )}}(x)\right\|_2 &= \left\|W^{(\ell )}\partial^\nu h^{(\ell)}_0\left(x, \varphi^{(\ell - 1)}\right)\right\|_2 \leq \left\|W^{(\ell)}\right\|_{2 \to 2} \left\|\partial^\nu h^{(\ell)}_0\left(x, \varphi^{(\ell - 1)}\right)\right\|_2 \\
        &\leq \kappa_\ell \left\|\partial^\nu h^{(\ell)}_0\left(x, \varphi^{(\ell - 1)}\right)\right\|_2
    \end{align*}
    The Fa\`a di Bruno formula in Theorem~\ref{prop:faa_di_bruno} together with the triangle inequality and Lemma~\ref{lemma:euclidean_norm_of_product_inequality} gives
    \begin{align*}
        &\left\|\partial^\nu h^{(\ell)}_0\left(x, \varphi^{(\ell - 1)}\right)\right\|_2 = \left\|\sum_{\pi \in \Pi([|\nu|])} h^{(\ell)}_{|\pi|}\left(x, \varphi^{(\ell - 1)}\right) \odot \left(\bigodot_{v \in \pi} \left(\partial^{\mu(\nu, v)}\cR_{\varphi^{(\ell - 1)}}(x)\right)\right)\right\|_2 \\
        &\qquad \leq \sum_{\pi \in \Pi([|\nu|])} \left\|h^{(\ell)}_{|\pi|}\left(x, \varphi^{(\ell - 1)}\right) \odot \left(\bigodot_{v \in \pi} \left(\partial^{\mu(\nu, v)}\cR_{\varphi^{(\ell - 1)}}(x)\right)\right)\right\|_2 \\
        &\qquad \leq \sum_{\pi \in \Pi([|\nu|])} \left(\sum_{i = 1}^{n_\ell}\left(\sigma^{(|\pi|)}\left(\cR_{\varphi^{(\ell - 1)}_i}(x)\right)\right)^2\prod_{v \in \pi} \left(\partial^{\mu(\nu, v)}\cR_{\varphi^{(\ell - 1)}_i} (x)\right)^2\right)^{1/2} \\
        & \qquad \leq \sum_{\pi \in \Pi([|\nu|])} A_{|\pi|} \left\|\bigodot_{v \in \pi} \left(\partial^{\mu(\nu, v)}\cR_{\varphi^{(\ell - 1)}}(x)\right)\right\|_2 \leq \sum_{\pi \in \Pi([|\nu|])} A_{|\pi|} \prod_{v \in \pi}\left\|\partial^{\mu(\nu, v)}\cR_{\varphi^{(\ell - 1)}}(x)\right\|_2 
    \end{align*}
    Combining everything and inserting the induction hypothesis we obtain
    \begin{align*}
        &\left\|\partial^\nu\cR_{\varphi^{(\ell )}}(x)\right\|_2 \leq \kappa_\ell \left\|\partial^\nu h^{(\ell)}_0\left(x, \varphi^{(\ell - 1)}\right)\right\|_2 \leq \kappa_\ell \sum_{\pi \in \Pi([|\nu|])} A_{|\pi|} \prod_{v \in \pi}\left\|\partial^{\mu(\nu, v)}\cR_{\varphi^{(\ell - 1)}}(x)\right\|_2  \\
        &\qquad \leq \kappa_\ell \sum_{\pi \in \Pi([|\nu|])} A_{|\pi|} \prod_{v \in \pi}\left(B^{(\ell - 1)}_{|\mu(\nu, v)|}(\kappa) \prod_{j = 1}^{n_0} \beta_j^{\left(\mu(\nu, v)\right)_j}\right) = \left(\kappa_\ell \sum_{\pi \in \Pi([|\nu|])} A_{|\pi|} \prod_{v \in \pi}B^{(\ell - 1)}_{|v|}(\kappa)\right) \prod_{j = 1}^{n_0} \beta_j^{\nu_j},
    \end{align*}
    where we used $|\mu(\nu, v)| = |v|$ for all $\pi \in \Pi([|\nu|])$ and all $v \in \pi$ as well as $\sum_{v\in \pi} \mu(\nu, v) = \nu$. 
\end{proof}
If we impose a more concrete bound on the derivatives of the activation function that is inspired by the bound from Lemma~\ref{lemma:derivative_bound_for_tanh}, then we obtain the following result.
\begin{corollary}\label{cor:exp_in_depth_derivative_bound}
    If there exists a constant $C_\sigma$ such that $\left|\sigma^{(r)}(x)\right| \leq A_r :=r! C_\sigma^{r - 1}$ holds for all $r \in \N$ and all $x \in \R$, then
    \begin{align*}
        \left\|D^u\cR_{\varphi^{(\ell)}}(x)\right\|_2 \leq \left(\prod_{m = 1}^\ell \kappa_m\right) \left(|u|!\right)\left(C_\sigma\sum_{p = 0}^{\ell - 1}\prod_{t = 1}^p\kappa_t\right)^{|u| - 1} \prod_{j \in u} \beta_j,
    \end{align*}
    holds for all $n_0, L \in \N$, all $\ell \in \{0, \ldots, L\}$, all $\varphi^{(\ell)} \in \Theta^{\mat}_\ell(n_0, \beta, \kappa)$, all non-empty $u \subseteq \left[n_0\right]$ and all $x \in \R^{n_0}$.
\end{corollary}
\begin{proof}
    We prove the statement using induction over $\ell \in \N_0$. For $\ell = 0$ and $q = 1$ we have $B^{(0)}_1(\kappa) = 1 \leq 1 = 1 (1!)0^0$. For $\ell = 0$ and $q > 1$ we have $B^{(0)}_q(\kappa) = 0 \leq 0 = 1 (q!)0^{q - 1}$.

    For the induction step let us suppose $\ell \in \N$ and the induction hypothesis is satisfied for $\ell - 1$ and every $q \in \N$. By Lemma \ref{lemma:sum_prod_of_factorial} and the binomial formula we have
    \begin{align*}
        B^{(\ell)}_q(\kappa) &= \kappa_\ell \sum_{\pi \in \Pi([q])} A_{|\pi|} \prod_{v \in \pi} B^{(\ell - 1)}_{|v|}(\kappa) =  \kappa_\ell \sum_{r = 1}^{q} A_{r} \sum_{\pi \in \Pi_r([q])}  \prod_{v \in \pi} B^{(\ell - 1)}_{|v|}(\kappa) \\
        &\leq \kappa_\ell  \sum_{r = 1}^{q} r! C_\sigma^{r - 1} \sum_{\pi \in \Pi_r([q])}  \prod_{v \in \pi} \left(\left(\prod_{m = 1}^{\ell - 1} \kappa_m\right) (|v|!) \left(C_\sigma \sum_{p = 0}^{\ell - 2} \prod_{t = 1}^p \kappa_t\right)^{|v| - 1}\right) \\
        &= \left(\prod_{m = 1}^{\ell} \kappa_m\right) \sum_{r = 1}^{q} r! \left(C_\sigma\prod_{m = 1}^{\ell - 1} \kappa_m\right)^{r - 1} \left(C_\sigma \sum_{p = 0}^{\ell - 2} \prod_{t = 1}^p \kappa_t\right)^{q- r}\sum_{\pi \in \Pi_r([q])}  \prod_{v \in \pi} |v|! \\
        &= \left(\prod_{m = 1}^{\ell} \kappa_m\right) (q!) \sum_{r = 1}^{q} \binom{q - 1}{r - 1} \left(C_\sigma\prod_{m = 1}^{\ell - 1} \kappa_m\right)^{r - 1} \left(C_\sigma \sum_{p = 0}^{\ell - 2} \prod_{t = 1}^p \kappa_t\right)^{q- r} \\
        &= \left(\prod_{m = 1}^{\ell} \kappa_m\right) (q!) \sum_{s = 0}^{q - 1} \binom{q - 1}{s} \left(C_\sigma\prod_{m = 1}^{\ell - 1} \kappa_m\right)^{s} \left(C_\sigma \sum_{p = 0}^{\ell - 2} \prod_{t = 1}^p \kappa_t\right)^{(q - 1) - s} \\
        &= \left(\prod_{m = 1}^\ell \kappa_m\right) (q!) \left(C_\sigma \sum_{p = 0}^{\ell - 1} \prod_{t = 1}^p \kappa_t\right)^{q - 1}
    \end{align*}
    An application of Theorem \ref{thm:derivative_bound_from_matrix_norm_bound} concludes the proof.
\end{proof}

If $\kappa_1 = \cdots = \kappa_L > 1$, then the bound in Corollary \ref{cor:exp_in_depth_derivative_bound} contains the term $\prod_{\ell = 1}^L \kappa_\ell = \kappa_1^L$, which grows exponentially in the depth $L$ of the neural network. As a remedy we give an additional condition that allows us to get rid of this exponential growth in the depth $L$ of the neural network. For every non-empty $u \subseteq [n_0]$ and every $\ell \in [L]$ we define $G^{(\ell)}_u: \R^{n_0} \times \Theta_{\ell - 1} \to \R^{n_\ell}$ by
\begin{align*}
    G^{(\ell)}_{u}\left(x, \varphi^{(\ell - 1)}\right) := h^{(\ell)}_1\left(x, \varphi^{(\ell - 1)}\right) \odot D^u\cR_{\varphi^{(\ell - 1)}}(x).
\end{align*}
Given $\alpha = \left(\alpha_\ell\right)_{\ell \in [L]} \in [0, \infty)^L$ and $\gamma = \left(\gamma_\ell\right)_{\ell \in [L]} \in [0, \infty)^L$ we define $\Theta^{\tangent}_L(n_0, \beta, \alpha, \gamma)$ as the set of all $\varphi \in \Theta_L$ such that for all non-empty $u \subseteq [n_0]$, all $\ell \in [L]$ and all $x \in [0, 1]^{n_0}$ the inequality
\begin{align}\label{eq:tangent_event_inequality}
    \left\|W^{(\ell)} G^{(\ell)}_{u}\left(x, \varphi^{(\ell - 1)}\right)\right\|_2 &\leq \alpha_\ell \left\|G^{(\ell)}_{u}\left(x, \varphi^{(\ell - 1)}\right)\right\|_2 + \gamma_\ell \prod_{j \in u} \beta_j
\end{align}
holds. By Lemma~\ref{lemma:tangent_set_closed} the set $\Theta^{\tangent}_L(n_0, \beta, \alpha, \gamma)$ is closed in $\Theta_L$, and in particular measurable. We define $\Xi^{(0)}_1 := 1$ and $\Xi^{(0)}_k := 0$ for all $k \geq 2$. Recursively, we define 
\begin{align*}
    \Xi^{(\ell)}_k := \alpha_\ell A_1 \Xi^{(\ell - 1)}_{k} + \gamma_\ell  + \kappa_\ell \sum_{\pi \in \Pi([k]), \pi \neq \{[k]\}} A_{|\pi|} \prod_{v \in \pi} \Xi^{(\ell - 1)}_{|v|} 
\end{align*}
for all $k, \ell \in \N$. 
\begin{remark}
The assumption
$|\sigma^{(r)}|\le r!C_\sigma^{r-1}$ includes
$\|\sigma'\|_{L^\infty(\R)}\le1$. This normalization is relevant for the
depth dependence: the linear tangent contribution propagates with the
factors $\alpha_\ell\|\sigma'\|_\infty$. More generally, the same argument
applies if the assumed bound on contiguous products is imposed on these
combined factors.
\end{remark}
\begin{lemma}\label{lemma:recursive_bound_tagent_condition}
     If for every $r \in \N$ there exists $A_r > 0$ such that $\left|\sigma^{(r)}(x)\right| \leq A_r$ holds for all $x \in \R$, then 
    \begin{align*}
        \left\|D^u\cR_{\varphi^{(\ell)}}(x)\right\|_2 \leq \Xi^{(\ell)}_{|u|} \prod_{j \in u} \beta_j
    \end{align*}
    holds for all $L \in \N$, for all $\varphi \in \Theta^{\mat}_L(n_0, \beta, \kappa) \cap \Theta^{\tangent}_L(n_0, \beta, \alpha, \gamma)$, for all $\ell \in \{0, \ldots, L\}$, all non-empty $u \subseteq [n_0]$ and all $x \in [0, 1]^{n_0}$.
\end{lemma}
\begin{proof}
    The proof is by induction over $\ell \in \{0, \ldots, L\}$. For $\ell = 0$ the realization of $\varphi^{(\ell)}$ is an affine function by definition \eqref{eq:zero_depth_realization}. If $j \in [n_0]$ and $u = \{j\}$, then
    \begin{align*}
        \left\|D^u \cR_{\varphi^{(0)}}(x)\right\|_2 = \left\|W^{(0)} e_j\right\|_2 \leq \beta_j = \Xi^{(0)}_{1} \prod_{i \in u} \beta_i
    \end{align*}
    holds for all $x \in \R^{n_0}$. If $|u| > 1$ then $D^u\cR_{\varphi^{(0)}}(x) = 0$ for all $x \in \R^{n_0}$ and therefore
    \begin{align*}
        \left\|D^u\cR_{\varphi^{(0)}}(x)\right\|_2 = 0 = \Xi^{(0)}_{|u|} \prod_{j \in u} \beta_j.
    \end{align*}

    Let $\ell \in [L]$ and suppose the statement is true for $\ell - 1$. The Fa\`a di Bruno formula in Theorem~\ref{prop:faa_di_bruno} gives for all non-empty $u \subseteq [n_0]$ that
    \begin{align*}
        D^{u}\cR_{\varphi^{(\ell)}}(x) &= W^{(\ell)} \sum_{\pi \in \Pi(u)} h^{(\ell)}_{|\pi|}(x) \odot \left(\bigodot_{v \in \pi} D^v \cR_{\varphi^{(\ell - 1)}}(x) \right) \\
        &= W^{(\ell)} G^{(\ell)}_{u}\left(x, \varphi^{(\ell - 1)}\right) + W^{(\ell)} \sum_{\pi \in \Pi(u), \pi \neq \{u\}} h^{(\ell)}_{|\pi|}(x) \odot \left(\bigodot_{v \in \pi} D^v \cR_{\varphi^{(\ell - 1)}}(x) \right).
    \end{align*}
    Since
    \begin{align*}
        \left\|W^{(\ell)} G^{(\ell)}_{u}\left(x, \varphi^{(\ell - 1)}\right)\right\|_2 &\leq \alpha_\ell \left\|G^{(\ell)}_{u}\left(x, \varphi^{(\ell - 1)}\right)\right\|_2 + \gamma_\ell \prod_{j \in u} \beta_j \leq \alpha_\ell A_1 \left\|D^u\cR_{\varphi^{(\ell - 1)}}(x)\right\|_2 + \gamma_\ell \prod_{j \in u} \beta_j \\
        &\leq \left( \alpha_\ell A_1 \Xi^{(\ell - 1)}_{|u|} + \gamma_\ell\right) \prod_{j \in u} \beta_j,
    \end{align*}
    we obtain by Lemma~\ref{lemma:euclidean_norm_of_product_inequality} that
    \begin{align*}
        \left\|D^{u}\cR_{\varphi^{(\ell)}}(x)\right\|_2 &\leq \left( \alpha_\ell A_1 \Xi^{(\ell - 1)}_{|u|} + \gamma_\ell\right) \prod_{j \in u} \beta_j + \left\|W^{(\ell)}\right\|_{2 \to 2} \sum_{\pi \in \Pi(u), \pi \neq \{u\}} A_{|\pi|} \prod_{v \in \pi} \left\|D^v \cR_{\varphi^{(\ell - 1)}}(x)\right\|_2 \\
        &\leq \left( \alpha_\ell A_1 \Xi^{(\ell - 1)}_{|u|} + \gamma_\ell\right) \prod_{j \in u} \beta_j + \kappa_\ell \sum_{\pi \in \Pi(u), \pi \neq \{u\}} A_{|\pi|} \prod_{v \in \pi} \left( \Xi^{(\ell - 1)}_{|v|} \prod_{j \in v} \beta_j\right) = \Xi^{(\ell)}_{|u|} \prod_{j \in u} \beta_j.
    \end{align*}
\end{proof}
\begin{proposition}\label{prop:explicit_bound_derivative}
    If $\overline{\kappa}, \overline{\alpha} \geq 1$ and there exists $C_\sigma > 0$ such that $\left|\sigma^{(r)}(x)\right| \leq A_r := r!C_\sigma^{r - 1}$ holds for all $r \in \N$ and all $x \in \R$, then there exist constants $C_0, C_1 > 0$ such that the following holds. If $L \in \N$ and
    \begin{align*}
        \kappa_\ell \leq \overline{\kappa}, \qquad \prod_{n = k}^m \alpha_n \leq \overline{\alpha} \quad \text{and} \quad 0 \leq \gamma_\ell \leq \overline{\kappa}/L
    \end{align*}
    holds for all $k, \ell, m \in [L]$ with $k \leq m$, then
    \begin{align*}
        \left\|D^u\cR_{\varphi^{(\ell)}}(x)\right\|_2 \leq C_0 \left(|u|\right)! \left(C_1 L\right)^{|u| - 1} \prod_{j \in u} \beta_j
    \end{align*}
    for all $\varphi \in \Theta^{\mat}_L(n_0, \beta, \kappa) \cap \Theta^{\tangent}_L(n_0, \beta, \alpha, \gamma)$, all $\ell \in \{0, \ldots, L\}$, all non-empty $u \subseteq [n_0]$ and all $x \in [0, 1]^{n_0}$.
\end{proposition}
\begin{proof}
    We define $C_0 := \max\left\{1, 4\overline{\alpha}(1 + 2\overline{\kappa})\right\}$ and 
    \begin{align*}
        c := \min \left\{1, \frac{1}{2C_\sigma C_0}, \frac{1}{4\overline{\alpha}\overline{\kappa} C_\sigma C_0}\right\}
    \end{align*}
    as well as $C_1 := 1/c$. Those constants are the right choices for the following arguments. Let $L \in \N$ and suppose $0 \leq \gamma_\ell \leq \overline{\kappa}/L$ holds for all $\ell \in [L]$. For every $\ell \in \{0, \ldots, L\}$ we define a function $F_\ell: \C \to \C$ by
    \begin{align*}
        F_\ell(z) := \sum_{k = 1}^{n_0} \frac{\Xi^{(\ell)}_k}{k!} z^k
    \end{align*}
    By the recursive definition, $\Xi_k^{(\ell)}\geq0$ for all
$k\in\N$ and $\ell\in\{0,\ldots,L\}$. Consequently, all coefficients
of $F_\ell$ and of its positive integer powers are non-negative. For all $k \in [n_0]$ and all $r \in \N$ we have
    \begin{align*}
        \coeff_k\left( \left(F_\ell(z)\right)^r\right) = \frac{r!}{k!} \sum_{{\pi \in \Pi([k])}, |\pi| = r} \prod_{v \in \pi} \Xi^{(\ell)}_{|v|}.
    \end{align*}
    Indeed,
\[
\begin{aligned}
k!\,\coeff_k\!\left(F_\ell(z)^r\right)
&=
\sum_{\substack{j_1,\ldots,j_r\geq1\\
j_1+\cdots+j_r=k}}
\frac{k!}{j_1!\cdots j_r!}
\prod_{m=1}^r\Xi^{(\ell)}_{j_m}.
\end{aligned}
\]
The right-hand side is the sum over ordered partitions of $[k]$ into
$r$ non-empty blocks, weighted by the product of the corresponding
$\Xi^{(\ell)}$-terms. Every unordered partition in $\Pi_r([k])$ has
exactly $r!$ orderings. Hence, we obtain
    \begin{align*}
        \sum_{\pi \in \Pi([k]), |\pi| \geq 2} A_{|\pi|} \prod_{v \in \pi} \Xi^{(\ell)}_{|v|} &= k! \sum_{r = 2}^\infty \frac{A_r}{k!} \sum_{\pi \in \Pi([k]), |\pi| = r} \prod_{v \in \pi} \Xi^{(\ell)}_{|v|} \leq k! \sum_{r = 2}^\infty C_\sigma^{r - 1} \left(\frac{r!}{k!} \sum_{\pi \in \Pi([k]), |\pi| = r} \prod_{v \in \pi} \Xi^{(\ell)}_{|v|} \right)\\
        &= k! \sum_{r = 2}^\infty C_\sigma^{r - 1} \coeff_k\left(\left(F_\ell(z)\right)^r\right) = k! \coeff_k\left( \sum_{r = 2}^\infty C_\sigma^{r - 1}  \left(F_\ell(z)\right)^r\right).
    \end{align*}
    Notice that the truncation of $F_\ell$ at degree $n_0$ causes no loss: for $k\leq n_0$, the coefficient of $z^k$ in any power $F_\ell^r$ depends only on coefficients of degrees at most $k$. We set $y := c/L$ and claim that $F_\ell(y) \leq C_0 y$ holds for all $\ell \in \{0, \ldots, L\}$. Looking at the definitions of $F_0$ and $\Xi^{(0)}_k$ for $k \in \N_0$ we immediately observe that $F_0(y) = y \leq C_0 y$. This is the base case. 

    Assume now that $\ell \in \N$ and $F_{\ell - 1}(y) \leq C_0 y$ holds. From this we obtain that
    \begin{align*}
        C_\sigma F_{\ell - 1}(y) \leq C_\sigma C_0 y \leq C_\sigma C_0 c \leq \frac{C_\sigma C_0}{2C_\sigma C_0} = \frac{1}{2}
    \end{align*}
    and therefore
    \begin{align*}
        \sum_{r = 2}^\infty C_\sigma^{r - 1}  \left(F_{\ell - 1}(y)\right)^r = \frac{C_\sigma \left(F_{\ell - 1}(y)\right)^2}{1 - C_\sigma F_{\ell - 1}(y)} \leq 2 C_\sigma \left(F_{\ell - 1}(y)\right)^2 \leq 2 C_\sigma C_0^2 y^2.
    \end{align*}
    Moreover, since $0 < c \leq 1$ and $L \geq 1$ we have $0 < y \leq 1$ and therefore $\exp(y) - 1 \leq 2y$. Consequently, we have
    \begin{align*}
        F_{\ell}(y) &= 
        \sum_{k = 1}^{n_0} \frac{\Xi^{(\ell)}_k}{k!} y^k = \sum_{k = 1}^{n_0} \frac{y^k}{k!} \left(\alpha_\ell A_1 \Xi^{(\ell - 1)}_{k} + \gamma_\ell  + \kappa_\ell \sum_{\pi \in \Pi([k]), \pi \neq \{[k]\}} A_{|\pi|} \prod_{v \in \pi} \Xi^{(\ell - 1)}_{|v|}\right) \\
        &\leq \sum_{k = 1}^{n_0} \frac{y^k}{k!} \left(\alpha_\ell \Xi^{(\ell - 1)}_{k} + \frac{\overline{\kappa}}{L}  + \overline{\kappa} \sum_{\pi \in \Pi([k]), \pi \neq \{[k]\}} A_{|\pi|} \prod_{v \in \pi} \Xi^{(\ell - 1)}_{|v|}\right) \\
        &\leq \sum_{k = 1}^{n_0} \frac{y^k}{k!} \left(\alpha_\ell \Xi^{(\ell - 1)}_{k} + \frac{\overline{\kappa}}{L}  + \overline{\kappa} k! \coeff_k\left( \sum_{r = 2}^\infty C_\sigma^{r - 1}  \left(F_{\ell - 1}(z)\right)^r \right) \right)\\
        &\leq \alpha_\ell F_{\ell - 1}(y) + \frac{\overline{\kappa}}{L} \sum_{k = 1}^{n_0} \frac{y^k}{k!} + \overline{\kappa} \sum_{r = 2}^\infty C_\sigma^{r - 1} \left(F_{\ell - 1}(y)\right)^r \leq \alpha_\ell F_{\ell - 1}(y) + \frac{\overline{\kappa}}{L} \left(\exp(y) - 1\right) + \overline{\kappa} 2 C_\sigma C_0^2 y^2 \\
        &\leq \alpha_\ell F_{\ell - 1}(y) + \frac{\overline{\kappa}}{L} 2y + \overline{\kappa} 2 C_\sigma C_0^2 y^2.  
    \end{align*}
    In the fourth line we used that every coefficient of $F_{\ell-1}$, and hence of every power $F_{\ell-1}^r$, is non-negative. Therefore the sum of the first $n_0$ coefficient contributions evaluated at $y>0$ is bounded by the full power series evaluated at $y$.
    If we define $p_\ell := F_\ell(y)/y$, then we have $p_0 = 1$ and 
    \begin{align*}
        p_\ell \leq \alpha_\ell p_{\ell - 1} + \left(\frac{2\overline{\kappa}}{L} + 2\overline{\kappa} C_\sigma C_0^2y\right).
    \end{align*}
    By Proposition~\ref{prop:linear_recurrence_relation} we have
    \begin{align*}
        p_\ell &\leq \prod_{j = 1}^\ell \alpha_j + \left(\frac{2\overline{\kappa}}{L} + 2\overline{\kappa} C_\sigma C_0^2y\right) \sum_{q = 1}^\ell \left(\prod_{j = q + 1}^\ell \alpha_j\right) \leq \overline{\alpha} + \left(\frac{2\overline{\kappa}}{L} + 2\overline{\kappa} C_\sigma C_0^2y\right) \ell \overline{\alpha} \\
        &\leq \overline{\alpha}(1 + 2\overline{\kappa}) + 2\overline{\alpha}\overline{\kappa} C_\sigma C_0^2 c. 
    \end{align*}
    By definition we have $\overline{\alpha}(1 + 2\overline{\kappa}) \leq C_0/4$ and $2\overline{\alpha}\overline{\kappa} C_\sigma C_0^2 c \leq C_0/2$ and we obtain
    \begin{align*}
        \frac{F_{\ell}(y)}{y} = p_\ell \leq \overline{\alpha}(1 + 2\overline{\kappa}) + 2\overline{\alpha}\overline{\kappa} C_\sigma C_0^2 c \leq \frac{C_0}{4} + \frac{C_0}{2} = \frac{3 C_0}{4} \leq C_0
    \end{align*}
    and $F_\ell(y) \leq C_0y$. Since $y$ and all coefficients of $F_\ell$ are non-negative we obtain 
    \begin{align*}
        \frac{\Xi^{(\ell)}_k}{k!} y^k \leq \sum_{j = 1}^{n_0} \frac{\Xi^{(\ell)}_j}{j!} y^j = F_\ell(y) \leq C_0 y.
    \end{align*}
    Hence,
    \begin{align*}
        \Xi^{(\ell)}_k \leq C_0 k! y^{1 - k} = C_0 k! \left(\frac{L}{c}\right)^{k - 1} = C_0 k!\left(C_1L\right)^{k - 1}
    \end{align*}
    An application of Lemma~\ref{lemma:recursive_bound_tagent_condition} concludes the proof.
\end{proof}

\begin{remark}[Majorant-series interpretation]
The generating function $F_\ell$ is a majorant series for the square-free mixed derivatives considered in this paper. Evaluating it at $y=c/L$ and using $F_\ell(y)\leq C_0y$ yields
\[
    \Xi_k^{(\ell)}
    \leq
    C_0 k! y^{1-k}
    =
    C_0 k!(C_1L)^{k-1}.
\]
This resembles a Cauchy-type coefficient estimate and explains the factor $L^{k-1}$. We stress, however, that the argument is a majorant-series argument in the derivative order. Since the theorem only controls square-free mixed derivatives, it does not by itself assert a holomorphic extension of the network realization to a complex polydisc.
\end{remark}

\section{High probability bounds for random neural networks}

In this section we show that the matrix and tangent conditions introduced above hold with high probability for sufficiently wide Gaussian neural networks. We fix one probability space $(\Omega, \fF, \Prob)$ and we assume that it is rich enough to accommodate all random variables we introduce. For all $\ell \in \N_0$, all $i \in [n_{\ell + 1}]$ and all $j \in [n_\ell]$ we assume that $X^{(\ell)}_{i, j}$ are i.i.d. standard normal random variables. For all $j \in [n_0]$ let $\tau^{(0)}_j > 0$ and define $\cW^{(0)}_{i, j} := \tau^{(0)}_j X^{(0)}_{i, j}$ for all $i \in [n_1]$. For all $\ell \in \N$ let $\tau_\ell > 0$ and define $\cW^{(\ell)} := \tau_\ell X^{(\ell)}$. For every $\ell \in \N_0$ we introduce a random variable $\Phi^{(\ell)}: \Omega \to \Theta_\ell$ by $\Phi^{(\ell)}(\omega) := \left(\cW^{(0)}(\omega), \ldots, \cW^{(\ell)}(\omega)\right)$, which has values in the space of neural networks. If it is clear that the neural network has depth $L \in \N$, then we just write $\Phi: \Omega \to \Theta_L$.

For all $\eta \in (0, 1)$, all $L \in \N$ and all $\ell \in \{1, \ldots, L\}$ we define
\begin{align*}
    \kappa_\ell(\eta,L) := \tau_\ell  \left(\sqrt{n_\ell} + \sqrt{n_{\ell + 1}} + \left(2\log\left(\frac{8L}{\eta}\right)\right)^{1/2}\right).
\end{align*}
For all $n_0 \in \N$, all $j \in [n_0]$ and all $\eta \in (0, 1)$ we define
\begin{align*}
    \beta_j(\eta,n_0) := \tau^{(0)}_j\left(n_1 + 2\sqrt{n_1\log\left(\frac{4n_0}{\eta}\right)} + 2\log\left(\frac{4n_0}{\eta}\right)\right)^{1/2}.
\end{align*}
For $\eta \in (0, 1)$ and $n_0, L \in \N$ we define the event $\cE_{\mat}(\eta, L, n_0)$ as the set of all $\omega \in \Omega$ such that $\Phi(\omega) \in \Theta^{\mat}_L\left(n_0, \beta(\eta, n_0), \kappa(\eta, L)\right)$.

\begin{proposition}\label{prop:matrix_norm_bound_high_prob}
    For all $\eta \in (0, 1)$ we have $\Prob\left(\cE_{\mat}(\eta, L, n_0)\right) \geq 1 - \eta/2$.
\end{proposition}
\begin{proof}
    Fix one $\ell \in \{1, \ldots, L\}$. We apply Proposition~\ref{prop:singular_values_gaussian_matrix} to the matrix $\cW^{(\ell)}$ with $t = \left(2\log\left(8L/{\eta}\right)\right)^{1/2}$. We obtain that 
    \begin{align*}
        \left\|\cW^{(\ell)}\right\|_{2 \to 2} \leq \tau_\ell \left(\sqrt{n_\ell} + \sqrt{n_{\ell + 1}} + \left(2\log\left(\frac{8L}{\eta}\right)\right)^{1/2}\right) = \kappa_\ell(\eta, L)
    \end{align*}
    with probability at least $1 - 2 \exp\left(-t^2/2\right) = 1 - 2 \eta/{8 L} = 1 - \eta/{4 L}$. The probability that this happens for all $\ell = 1, \ldots, L$ simultaneously is at least $1 - L \eta/{4L} = 1 - \eta/{4}$. For the first-layer column bounds let us fix $j \in \{1, \ldots, n_0\}$. For every $i = 1, \ldots, n_1$ the random variable $\cW^{(0)}_{i, j}$ is Gaussian and by Lemma~\ref{lemma:chi_squared_sum_concentration} with $x = \log(4n_0/\eta)$ and $\sqrt{a_i} = \tau^{(0)}_{j}$ for all $i \in \{1, \ldots, n_1\}$ we obtain
    \begin{align*}
        \Prob\left(\sum_{i = 1}^{n_1} (\tau^{(0)}_j)^2\left(\left(X^{(0)}_{i, j}\right)^2 - 1\right) \geq 2 (\tau^{(0)}_j)^2 \left(n_1 \log\left(\frac{4n_0}{\eta}\right)\right)^{1/2} + 2(\tau^{(0)}_j)^2\log\left(\frac{4n_0}{\eta}\right)\right) \leq \frac{\eta}{4n_0}.
    \end{align*}
    From this we deduce that $\|\cW^{(0)} e_j\|_2 \leq \beta_j(\eta, n_0)$ holds with probability at least $1 - \eta/(4n_0)$. The probability that this happens for all $j \in [n_0]$ simultaneously is at least $1 - \eta/4$. If we want all the bounds simultaneously we get a probability of at least $1 - \eta/2$.
\end{proof}

Proposition~\ref{prop:matrix_norm_bound_high_prob} together with Theorem~\ref{thm:derivative_bound_from_matrix_norm_bound} immediately gives Lemma~\ref{lemma:derivative_bell_bound_matrix_event}.

\begin{lemma}\label{lemma:derivative_bell_bound_matrix_event}
    Suppose that for every $r \in \N$ there exists $A_r > 0$ such that $\left|\sigma^{(r)}(x)\right| \leq A_r$ holds for all $x \in \R$. If $L, n_0 \in \N$ and $\eta \in (0, 1)$, then 
    \begin{align*}
        \left\|\partial^\nu\cR_{\Phi^{(\ell)}}(x)\right\|_2 \leq B^{(\ell)}_{|\nu|}\left(\kappa(\eta, L)\right) \prod_{j = 1}^{n_0} \left(\beta_j(\eta, n_0)\right)^{\nu_j},
    \end{align*}
    holds on the event $\cE_{\mat}(\eta, L, n_0)$, and in particular it holds at least with probability $1 - \eta/2$, for all $\nu\in\N_0^{n_0}\setminus\{0\}$, all $x\in\R^{n_0}$, and all $\ell=0,\ldots,L$.
\end{lemma}
Combining Lemma~\ref{lemma:derivative_bound_for_tanh} with Corollary~\ref{cor:exp_in_depth_derivative_bound} gives the following result.
\begin{corollary}\label{cor:probabilistic_exp_in_depth_derivative_bound}
    Suppose there exists a constant $C_\sigma$ such that $\left|\sigma^{(r)}(x)\right| \leq A_r :=r! C_\sigma^{r - 1}$ holds for all $r \in \N$ and all $x \in \R$. If $L, n_0 \in \N$ and $\eta \in (0, 1)$, then
    \begin{align*}
        \left\|D^u\cR_{\Phi^{(\ell)}}(x)\right\|_2 \leq \left(\prod_{m = 1}^\ell \kappa_m(\eta, L)\right) \left(|u|!\right)\left(C_\sigma\sum_{p = 0}^{\ell - 1}\prod_{t = 1}^p\kappa_t(\eta, L)\right)^{|u| - 1} \prod_{j \in u} \beta_j(\eta, n_0),
    \end{align*}
    holds on the event $\cE_{\mat}(\eta, L, n_0)$, and in particular it holds at least with probability $1 - \eta/2$, for all non-empty $u\subseteq[n_0]$, all $x\in\R^{n_0}$, and all $\ell=0,\ldots,L$.
\end{corollary}

Let us for a moment assume that $n_1 = n_2 = \cdots = n_{L + 1} = n$. For the Xavier initialization we obtain
\begin{align*}
    \kappa_\ell(\eta,L) := \frac{\sqrt{2}}{\sqrt{2n}}  \left(2\sqrt{n} + \left(2\log\left(\frac{8L}{\eta}\right)\right)^{1/2}\right) \geq 2
\end{align*}
for all $\ell \in [L]$. This means that the bound from Corollary~\ref{cor:probabilistic_exp_in_depth_derivative_bound} grows exponentially with the depth $L$. The goal of this section is to get a better bound with high probability. We will see that this is possible at least for very wide neural networks. We assume that $\beta_j > 0$ holds for all $j \in [n_0]$. For $L \in \N$ and $\ell \in [L]$ we define
\begin{align*}
    \cA^{(\ell)}_u(v, L, \beta) := \left\{x \in [0, 1]^{n_0} : L\left\|G^{(\ell)}_u(x, v)\right\|_2 \geq \prod_{j \in u} \beta_j\right\},
\end{align*}
where $\emptyset \neq u \subseteq [n_0]$ and $v \in \Theta_{\ell - 1}$. We define
\begin{align*}
    \cT^{(\ell)}_u\left(v, L, \beta\right):= \left\{\left\| G^{(\ell)}_{u}\left(x, v\right)\right\|_2^{-1} G^{(\ell)}_{u}\left(x, v\right) \mid x \in  \cA^{(\ell)}_u(v, L, \beta)\right\}  \subseteq \sphere^{n_{\ell} - 1}
\end{align*}
and 
\begin{align*}
    \cT_\ell(v, L, \beta) := \bigcup_{\emptyset \neq u \subseteq [n_0]} \cT^{(\ell)}_u(v, L, \beta).
\end{align*}
Furthermore, we set
\begin{align*}
    M^{(\ell)}_k(\kappa) := A_2 B_1^{(\ell-1)}(\kappa)B_k^{(\ell-1)}(\kappa) + A_1B^{(\ell - 1)}_{k + 1}(\kappa).
\end{align*}
\begin{proposition}\label{prop:measurable_tangent_net}
    If $\varepsilon > 0$ and $L, n_0 \in \N$, then for every $\ell \in [L]$ there exists a finite set that we denote by $\cJ_\ell\left(\varepsilon, L, n_0, \beta, \kappa\right)$, with
    \begin{align*}
         \left|\cJ_\ell\left(\varepsilon, L, n_0, \beta, \kappa\right)\right| \leq \sum_{k = 1}^{n_0} \binom{n_0}{k} \prod_{j = 1}^{n_0} \left(1 + \frac{2n_0 L M^{(\ell)}_k(\kappa) \beta_j}{\varepsilon}\right) =: \exp\left(\fH_\ell(\varepsilon, L, n_0, \kappa, \beta)\right)
    \end{align*}
    and for every $i \in \cJ_\ell\left(\varepsilon, L, n_0, \beta, \kappa\right)$ there is a Borel measurable map $\xi^{(\ell)}_{\varepsilon, L, n_0, \beta, \kappa, i}: \Theta_{\ell - 1} \to \sphere^{n_\ell - 1}$ such that for every $v \in \Theta^{\mat}_{\ell - 1}(n_0, \beta, \kappa)$ and every $y \in \cT_\ell\left(v, L, \beta\right)$ there exists $i \in \cJ_\ell\left(\varepsilon, L, n_0, \beta, \kappa\right)$ such that 
    \begin{align*}
        \left\|y - \xi^{(\ell)}_{\varepsilon, L, n_0, \beta, \kappa, i}(v)\right\|_2 \leq \varepsilon.
    \end{align*}
\end{proposition}
\begin{proof}
    Fix $\ell \in [L]$ and a non-empty set $u \subseteq [n_0]$. For $j \in [n_0]$ we define
    \begin{align*}
        N^{(\ell)}_{u, j}(n_0, L, \kappa, \beta, \varepsilon) := \max \left\{1, \left\lceil \frac{2 n_0 L M^{(\ell)}_{|u|}(\kappa) \beta_j}{\varepsilon}\right\rceil\right\}
    \end{align*}
    and
    \begin{align*}
        I^{(\ell)}_{u, j, i}(n_0, L, \kappa, \beta, \varepsilon) := \left[\frac{i - 1}{N^{(\ell)}_{u, j}(n_0, L, \kappa, \beta, \varepsilon)}, \frac{i}{N^{(\ell)}_{u, j}(n_0, L, \kappa, \beta, \varepsilon)}\right], \quad i = 1, \ldots, N^{(\ell)}_{u, j}(n_0, L, \kappa, \beta, \varepsilon). 
    \end{align*}
    If $k_j \in \left[N^{(\ell)}_{u, j}(n_0, L, \kappa, \beta, \varepsilon)\right]$ for every $j \in [n_0]$, then we define
    \begin{align*}
        R^{(\ell)}_{u, k}\left(n_0, L, \kappa, \beta, \varepsilon\right) := I^{(\ell)}_{u, 1, k_1}(n_0, L, \kappa, \beta, \varepsilon) \times \cdots \times I^{(\ell)}_{u, n_0, k_{n_0}}(n_0, L, \kappa, \beta, \varepsilon)
    \end{align*}
    and
    \begin{align*}
        \cQ^{(\ell)}_u(n_0, L, \kappa, \beta, \varepsilon) := \left\{R^{(\ell)}_{u, k}\left(n_0, L, \kappa, \beta, \varepsilon\right) \mid \forall j \in [n_0]: k_j \in \left[N^{(\ell)}_{u, j}(n_0, L, \kappa, \beta, \varepsilon)\right]\right\}
    \end{align*}
    is a set that covers $[0, 1]^{n_0}$. The set 
    \begin{align*}
        \cJ_\ell\left(\varepsilon, L, n_0, \beta, \kappa\right) := \left\{(u, R)\mid \emptyset \neq u \subseteq [n_0], R \in \cQ^{(\ell)}_u(n_0, L, \kappa, \beta, \varepsilon)\right\}
    \end{align*}
    satisfies 
    \begin{align*}
        \left|\cJ_\ell\left(\varepsilon, L, n_0, \beta, \kappa\right)\right| &= \sum_{\emptyset \neq u \subseteq[n_0]} \left|\cQ^{(\ell)}_u(n_0, L, \kappa, \beta, \varepsilon)\right| = \sum_{\emptyset \neq u \subseteq[n_0]} \prod_{j = 1}^{n_0}N^{(\ell)}_{u, j}(n_0, L, \kappa, \beta, \varepsilon) \\
        &\leq \sum_{\emptyset \neq u \subseteq[n_0]} \prod_{j = 1}^{n_0}\left(1 + \frac{2 n_0 L M^{(\ell)}_{|u|}(\kappa) \beta_j}{\varepsilon}\right) = \sum_{k = 1}^{n_0} \binom{n_0}{k} \prod_{j = 1}^{n_0}\left(1 + \frac{2 n_0 L M^{(\ell)}_{k}(\kappa) \beta_j}{\varepsilon}\right).
    \end{align*}
    It remains to construct the measurable maps. Fix $i = (u, R) \in \cJ_\ell\left(\varepsilon, L, n_0, \beta, \kappa\right)$ and define 
    \begin{align*}
        D^{(\ell)}_{u, R}\left(L, \beta, \kappa\right) := \left\{v \in \Theta_{\ell - 1}^{\mat}(n_0, \beta, \kappa)\mid \cA^{(\ell)}_u(v, L, \beta) \cap R \neq \emptyset\right\}.
    \end{align*}
    The set
    \begin{align*}
        \Gamma^{(\ell)}_{u, R} := \left\{(v, x) \in \Theta_{\ell - 1}^{\mat}(n_0, \beta, \kappa) \times R: L \left\|G^{(\ell)}_u(x, v)\right\|_2 \geq \prod_{j \in u}\beta_j\right\}
    \end{align*}
    is closed in $\Theta_{\ell - 1}^{\mat}(n_0, \beta, \kappa) \times R$, since $G^{(\ell)}_u: \Theta_{\ell - 1}^{\mat}(n_0, \beta, \kappa) \times R \to \R^{n_\ell}$ is a continuous map. Since $R$ is compact we can apply Lemma~\ref{lemma:measurable_lexicographic_representative} to obtain that the map 
    \begin{align*}
        q^{(\ell)}_{u, R, L, \beta, \kappa}: D^{(\ell)}_{u, R}\left(L, \beta, \kappa\right) \to R \quad \text{given by} \quad  q^{(\ell)}_{u, R, L, \beta, \kappa}(v) := \lexmin\left(\cA^{(\ell)}_u(v, L, \beta) \cap R\right)
    \end{align*}
    is Borel measurable, where $\lexmin$ is the minimum in the lexicographic order. The set
$\Theta_{\ell-1}^{\mat}(n_0,\beta,\kappa)$ is closed in
$\Theta_{\ell-1}$, since it is defined by finitely many continuous norm
inequalities. Hence, $D^{(\ell)}_{u,R}(L,\beta,\kappa)$, which is closed relative to $\Theta_{\ell-1}^{\mat}(n_0,\beta,\kappa)$, is a Borel subset of
$\Theta_{\ell-1}$. On $D^{(\ell)}_{u,R}(L,\beta,\kappa)$, the selected point belongs to the
active set, and therefore
\[
    \left\|G_u^{(\ell)}
    \left(q^{(\ell)}_{u,R,L,\beta,\kappa}(v),v\right)\right\|_2
    \geq
    \frac1L\prod_{j\in u}\beta_j>0.
\]
 Consequently, the map $\xi^{(\ell)}_{\varepsilon, L, n_0, \beta, \kappa, (u, R)}: \Theta_{\ell - 1} \to \sphere^{n_\ell - 1}$ given by
    \begin{align*}
        \xi^{(\ell)}_{\varepsilon, L, n_0, \beta, \kappa, (u, R)}(v) :=
        \begin{cases}
            \left\|G^{(\ell)}_u\left(q^{(\ell)}_{u, R, L, \beta, \kappa}(v), v\right)\right\|_2^{-1}G^{(\ell)}_u\left(q^{(\ell)}_{u, R, L, \beta, \kappa}(v), v\right) &, \text{if } v \in D^{(\ell)}_{u, R}\left(L, \beta, \kappa\right) \\
            e_1&, \text{otherwise}
        \end{cases}
    \end{align*}
    is Borel measurable. If $v \in \Theta_{\ell - 1}^{\mat}(n_0, \beta, \kappa)$ and $y \in \cT_\ell(v, L, \beta)$, then there exists a non-empty $u \subseteq [n_0]$ and a point $x \in \cA^{(\ell)}_u(v, L, \beta)$ such that $y =  \left\|G^{(\ell)}_u\left(x, v\right)\right\|_2^{-1}G^{(\ell)}_u\left(x, v\right)$. Choosing $R \in \cQ^{(\ell)}_u(n_0, L, \kappa, \beta, \varepsilon)$ such that $x \in R$ we have by construction $q^{(\ell)}_{u, R, L, \beta, \kappa}(v) \in \cA^{(\ell)}_u(v, L, \beta) \cap R$. Lemma~\ref{lemma:lipschitz_bound_for_tangent_maps} gives
    \begin{align*}
        \left\|G_u^{(\ell)}(x,v)- G_u^{(\ell)}\left(q^{(\ell)}_{u, R, L, \beta, \kappa}(v),v\right)\right\|_2 &\leq M^{(\ell)}_{|u|}(\kappa) \prod_{j \in u}\beta_j \sum_{m = 1}^{n_0} \beta_m \left|x_m - \left(q^{(\ell)}_{u, R, L, \beta, \kappa}(v)\right)_m\right| \\
        &\leq M^{(\ell)}_{|u|}(\kappa) \prod_{j \in u}\beta_j \sum_{m = 1}^{n_0} \beta_m \left(N^{(\ell)}_{u, m}(n_0, L, \kappa, \beta, \varepsilon)\right)^{-1} \leq \frac{\varepsilon}{2L} \prod_{j \in u}\beta_j.
    \end{align*}
    By Lemma~\ref{lemma:stability_of_normalization} we have
    \begin{align*}
        \left\|y - \xi^{(\ell)}_{\varepsilon, L, n_0, \beta, \kappa, (u, R)}(v)\right\|_2 &= \left\|\frac{G^{(\ell)}_u\left(x, v\right)}{\left\|G^{(\ell)}_u\left(x, v\right)\right\|_2} - \frac{G^{(\ell)}_u\left(q^{(\ell)}_{u, R, L, \beta, \kappa}(v), v\right)}{\left\|G^{(\ell)}_u\left(q^{(\ell)}_{u, R, L, \beta, \kappa}(v), v\right)\right\|_2}\right\|_2 \\
        &\leq \frac{2\left\|G^{(\ell)}_u\left(x, v\right) - G^{(\ell)}_u\left(q^{(\ell)}_{u, R, L, \beta, \kappa}(v), v\right)\right\|_2}{\min\left\{\left\|G^{(\ell)}_u\left(x, v\right)\right\|_2, \left\|G^{(\ell)}_u\left(q^{(\ell)}_{u, R, L, \beta, \kappa}(v), v\right)\right\|_2\right\}} \leq \varepsilon.
    \end{align*}
    Indeed, both $x$ and $q^{(\ell)}_{u,R,L,\beta,\kappa}(v)$ belong to the active set, so each norm in the denominator is at least $L^{-1}\prod_{j\in u}\beta_j$.
\end{proof}
For $\eta \in (0, 1)$ and $n_0, L \in \N$ we define $\alpha_\ell(L) := \tau_\ell\sqrt{n_{\ell + 1}}\left(1 + 1/L\right)$ and $\gamma_\ell(\eta, L) := \kappa_\ell(\eta, L)/{L}$. We define the tangent event by
\[
    \cE_{\tangent}(\eta,L,n_0)
    :=
    \left\{\omega\in\Omega:
    \Phi(\omega)\in\Theta^{\tangent}_L\left(
        n_0,\beta(\eta,n_0),\alpha(L),\gamma(\eta,L)
    \right)
    \right\}.
\]
\begin{theorem}\label{thm:mat_and_tangent_event_high_probability}
    Let $\eta \in (0, 1)$ and $L, n_0 \in \N$. If 
    \begin{align*}
        n_{\ell + 1} \geq 8L^2\left(\fH_\ell\left(\frac{\tau_\ell \sqrt{n_{\ell + 1}}}{2L \kappa_\ell(\eta, L)}, L, n_0, \kappa(\eta,L), \beta(\eta, n_0)\right) + \log\left(\frac{2L}{\eta}\right)\right)
    \end{align*}
    holds for all $\ell \in [L]$, then $\Prob\left(\cE_{\mat}\left(\eta, L, n_0\right) \cap \cE_{\tangent}\left(\eta, L, n_0\right)\right) \geq 1 - \eta$.
\end{theorem}
\begin{proof}
Fix $\ell\in[L]$ and set $\varepsilon_\ell := \tau_\ell\sqrt{n_{\ell+1}}/{2L\kappa_\ell(\eta,L)}$. By Proposition~\ref{prop:measurable_tangent_net}, there is a finite index set $\cK_\ell(\eta,L,n_0) := \cJ_\ell\left(\varepsilon_\ell,L,n_0, \beta(\eta,n_0),\kappa(\eta,L) \right)$ satisfying
\[
    |\cK_\ell(\eta,L,n_0)|
    \leq
    \exp\left(
        \fH_\ell\left(
            \varepsilon_\ell,L,n_0,
            \kappa(\eta,L),\beta(\eta,n_0)
        \right)
    \right).
\]
For every $i\in\cK_\ell(\eta,L,n_0)$, let $\zeta^{(\ell)}_{\eta,L,n_0,i}:\Theta_{\ell-1} \longrightarrow \sphere^{n_\ell-1}$ be the measurable net map from Proposition~\ref{prop:measurable_tangent_net}. Thus, for every $v\in\Theta^{\mat}_{\ell-1} \left(n_0,\beta(\eta,n_0),\kappa(\eta,L)\right)$ and every $y\in\cT_\ell(v,L,\beta(\eta,n_0))$, there exists $i\in\cK_\ell(\eta,L,n_0)$ such that
\[
    \left\|y-\zeta^{(\ell)}_{\eta,L,n_0,i}(v)\right\|_2
    \leq\varepsilon_\ell.
\]

Define the measurable set
\[
\begin{aligned}
    D_\ell(\eta,L,n_0)
    :=
    \bigcup_{i\in\cK_\ell(\eta,L,n_0)}
    \Big\{(v,w):
    &\left\|w\zeta^{(\ell)}_{\eta,L,n_0,i}(v)\right\|_2 \geq
    \tau_\ell\sqrt{n_{\ell+1}}
    \left(1+\frac{1}{2L}\right)
    \Big\},
\end{aligned}
\]
where $(v,w)\in\Theta_{\ell-1}\times\R^{n_{\ell+1}\times n_\ell}$. Indeed, for each $i$, the map $(v,w)\longmapsto w\,\zeta^{(\ell)}_{\eta,L,n_0,i}(v)$ is Borel measurable, since
$v\mapsto\zeta^{(\ell)}_{\eta,L,n_0,i}(v)$ is Borel and matrix-vector
multiplication is continuous. Hence $D_\ell(\eta,L,n_0)$ is Borel as a
finite union of Borel sets. Let
\[
    \cB_\ell(\eta,L,n_0)
    :=
    \left\{\omega:
    \left(\Phi^{(\ell-1)}(\omega),\cW^{(\ell)}(\omega)\right)
    \in D_\ell(\eta,L,n_0)
    \right\}.
\]
We first estimate the probability of the corresponding event for a fixed previous-layer parameter $v$. Fix $i\in\cK_\ell(\eta,L,n_0)$ and set $z:=\zeta^{(\ell)}_{\eta,L,n_0,i}(v) \in \sphere^{n_\ell-1}$. For $r=1,\ldots,n_{\ell+1}$ define
\[
    Y_r:=\sum_{j=1}^{n_\ell}z_jX_{r,j}^{(\ell)}.
\]
Since $z$ is deterministic and has Euclidean norm one, the random variables $Y_1,\ldots,Y_{n_{\ell+1}}$ are independent standard normal random variables. Moreover,
\[
    \|\cW^{(\ell)}z\|_2^2
    =
    \tau_\ell^2\sum_{r=1}^{n_{\ell+1}}Y_r^2.
\]
Apply Lemma~\ref{lemma:chi_squared_sum_concentration} with $a_1=\cdots=a_{n_{\ell+1}}=1$ and $x = n_{\ell+1}/{(8L^2)}$. The corresponding deviation threshold is
\[
    \frac{n_{\ell+1}}{\sqrt2L}
    +
    \frac{n_{\ell+1}}{4L^2}
    \leq
    n_{\ell+1}
    \left[
        \left(1+\frac{1}{2L}\right)^2-1
    \right].
\]
Hence
\begin{align*}
    \Prob\left(
        \|\cW^{(\ell)}z\|_2
        \geq
        \tau_\ell\sqrt{n_{\ell+1}}
        \left(1+\frac{1}{2L}\right)
    \right)
    \leq
    \exp\left(-\frac{n_{\ell+1}}{8L^2}\right).
\end{align*}

Let $g_\ell(v) := \E\left[\chara_{D_\ell(\eta,L,n_0)} \left(v,\cW^{(\ell)}\right)\right]$. The union bound and the assumed width condition give
\[
\begin{aligned}
    g_\ell(v)
    &\leq
    |\cK_\ell(\eta,L,n_0)|
    \exp\left(-\frac{n_{\ell+1}}{8L^2}\right) \leq
    \exp\left(
        \fH_\ell\left(
            \varepsilon_\ell,L,n_0,
            \kappa(\eta,L),\beta(\eta,n_0)
        \right)
        -\frac{n_{\ell+1}}{8L^2}
    \right)
    \leq
    \frac{\eta}{2L}.
\end{aligned}
\]
Since $\Phi^{(\ell-1)}$ and $\cW^{(\ell)}$ are independent, Lemma~\ref{lemma:independent_cond_exp} yields $\Prob\left(\cB_\ell(\eta,L,n_0)\right) = \E\left[g_\ell\left(\Phi^{(\ell-1)}\right)\right] \leq \eta/(2L)$. Consequently,
\[
    \Prob\left(
        \bigcup_{\ell=1}^L\cB_\ell(\eta,L,n_0)
    \right)
    \leq\frac{\eta}{2}.
\]
Together with Proposition~\ref{prop:matrix_norm_bound_high_prob}, this shows that the event
\[
    \cG
    :=
    \cE_{\mat}(\eta,L,n_0)
    \cap
    \bigcap_{\ell=1}^L
    \left(\Omega\setminus\cB_\ell(\eta,L,n_0)\right)
\]
has probability at least $1-\eta$. It remains to show that $\cG\subseteq\cE_{\tangent}(\eta,L,n_0)$. Fix $\omega\in\cG$, $\ell\in[L]$, a non-empty set $u\subseteq[n_0]$, and $x\in[0,1]^{n_0}$. If
\[
    L\left\|G_u^{(\ell)}
    \left(x,\Phi^{(\ell-1)}(\omega)\right)\right\|_2
    <
    \prod_{j\in u}\beta_j(\eta,n_0),
\]
then the matrix event gives
\[
\begin{aligned}
    \left\|\cW^{(\ell)}(\omega)G_u^{(\ell)}(x,\Phi^{(\ell-1)}(\omega))\right\|_2
    &\leq
    \kappa_\ell(\eta,L)
    \left\|G_u^{(\ell)}(x,\Phi^{(\ell-1)}(\omega))\right\|_2 \leq
    \frac{\kappa_\ell(\eta,L)}{L}
    \prod_{j\in u}\beta_j(\eta,n_0).
\end{aligned}
\]
Otherwise, the normalized tangent
\[
    q := \frac{G_u^{(\ell)} \left(x,\Phi^{(\ell-1)}(\omega)\right)}
    {\left\|G_u^{(\ell)}
    \left(x,\Phi^{(\ell-1)}(\omega)\right)\right\|_2}
\]
belongs to $\cT_\ell(\Phi^{(\ell-1)}(\omega),L,\beta(\eta,n_0))$. Choose a net point $\zeta_i$ with $\|q-\zeta_i\|_2\leq\varepsilon_\ell$. Since $\omega\notin\cB_\ell(\eta,L,n_0)$ and $\omega\in\cE_{\mat}(\eta,L,n_0)$,
\[
\begin{aligned}
    \|\cW^{(\ell)}q\|_2
    &\leq
    \|\cW^{(\ell)}\zeta_i\|_2
    +
    \|\cW^{(\ell)}\|_{2\to2}\|q-\zeta_i\|_2  <
    \tau_\ell\sqrt{n_{\ell+1}}
    \left(1+\frac{1}{2L}\right)
    +
    \kappa_\ell(\eta,L)\varepsilon_\ell \\
    &=
    \tau_\ell\sqrt{n_{\ell+1}}
    \left(1+\frac1L\right).
\end{aligned}
\]
Multiplication by the tangent norm proves~\eqref{eq:tangent_event_inequality} with $\alpha_\ell(L) = \tau_\ell\sqrt{n_{\ell+1}}\left(1+ 1/L\right)$ and $\gamma_\ell(\eta,L) = \kappa_\ell(\eta,L)/{L}$. Thus, $\omega\in\cE_{\tangent}(\eta,L,n_0)$, which completes the proof.
\end{proof}

\begin{theorem}\label{thm:main_high_probability_derivative_bound}
    If $\overline{\kappa}, \overline{\alpha} \geq 1$ and there exists $C_\sigma > 0$ such that  $\left|\sigma^{(r)}(x)\right| \leq A_r :=r! C_\sigma^{r - 1}$ holds for all $r \in \N$ and all $x \in \R$, then there exist constants $C_0, C_1 > 0$ such that the following holds.
    If $\eta \in (0, 1)$ and $L, n_0 \in \N$ are such that
    \begin{enumerate}
        \item for all $\ell \in [L]$ the inequality $\kappa_\ell(\eta, L) \leq \overline{\kappa}$ holds,
        \item for all $1 \leq k \leq m \leq L$ the inequality 
            \begin{align*}
                \prod_{\ell = k}^m \alpha_\ell(L) = \left(1 + \frac{1}{L}\right)^{m - k + 1} \prod_{\ell = k}^m \tau_\ell \sqrt{n_{\ell + 1}} \leq \overline{\alpha}
            \end{align*}
            holds, and
        \item for all $\ell \in [L]$ the inequality
        \begin{align*}
            n_{\ell + 1} \geq 8L^2\left(\fH_\ell\left(\frac{\tau_\ell \sqrt{n_{\ell + 1}}}{2 L \kappa_\ell(\eta, L)}, L, n_0, \kappa(\eta, L), \beta(\eta, n_0)\right) + \log\left(\frac{2L}{\eta}\right)\right)
        \end{align*}
        
    \end{enumerate}
    holds, then 
    \begin{align*}
        \Prob\left(\bigcap_{\emptyset \neq u \subseteq [n_0]} \left\{\sup_{x \in [0, 1]^{n_0}}\left\|D^u \cR_{\Phi^{(L)}}(x)\right\|_2 \leq C_0 |u|! \left(C_1 L\right)^{|u| - 1} \prod_{j \in u} \beta_j(\eta, n_0)\right\}\right) \geq 1 - \eta.
    \end{align*}
\end{theorem}
\begin{proof}
    By Proposition~\ref{prop:explicit_bound_derivative} we obtain constants $C_0, C_1 > 0$ depending on $\overline{\kappa}, \overline{\alpha} \geq 1$ and $C_\sigma > 0$. Let $\eta \in (0, 1)$ and $L, n_0 \in \N$ satisfying the conditions of this theorem. For all $\ell\in[L]$, we have $ 0 \leq \gamma_\ell(\eta, L) = \frac{\kappa_\ell(\eta, L)}{L} \leq \overline{\kappa}/{L}$ and 
    \begin{align*}
        \prod_{\ell = k}^m \alpha_\ell(L) = \left(1 + \frac{1}{L}\right)^{m - k + 1} \prod_{\ell = k}^m \tau_\ell \sqrt{n_{\ell + 1}} \leq \overline{\alpha}
    \end{align*}
    For $\omega \in \cE_{\mat}\left(\eta, L, n_0\right) \cap \cE_{\tangent}\left(\eta, L ,n_0\right)$ we have 
    \begin{align*}
        \Phi^{(L)}(\omega) \in \Theta_{L}^{\mat}\left(n_0, \beta(\eta, n_0), \kappa(\eta, L)\right) \cap \Theta_{L}^{\tangent}\left(n_0, \beta(\eta,n_0), \alpha(L), \gamma(\eta, L)\right)
    \end{align*}
    and therefore
    \begin{align*}
        \left\|D^u \cR_{\Phi^{(L)}(\omega)}(x)\right\|_2 \leq C_0 |u|! \left(C_1 L\right)^{|u| - 1} \prod_{j \in u} \beta_j(\eta, n_0)
    \end{align*}
    holds for all non-empty $u \subseteq [n_0]$ and all $x \in [0, 1]^{n_0}$. Since we have
    \begin{align*}
         \Prob\left(\cE_{\mat}\left(\eta, L, n_0\right) \cap \cE_{\tangent}\left(\eta, L, n_0\right)\right) \geq 1 - \eta,
    \end{align*}
    by Theorem~\ref{thm:mat_and_tangent_event_high_probability} the claim follows.
\end{proof}

\begin{proposition}\label{prop:xavier_scalar_readout}
    For the activation $\sigma = \tanh$, there exist constants $C, C_0, C_1 > 0$ with the following property. Let $\eta \in (0, 1)$, let $L \in \N\setminus\{1\}$, let $n_1 = \cdots = n_L = n$ and $n_{ L + 1} = 1$. Assume Xavier initialization, that is, 
    \begin{align*}
        \tau^{(0)}_j = \sqrt{\frac{2}{n_0 + n}} , \quad \text{for all } j \in [n_0], \quad  \tau_\ell^2 = \frac{1}{n}, \quad \text{for all } \ell \in [L - 1] \quad \text{and} \quad \tau_L^2 =\frac{2}{n + 1}. 
    \end{align*}
    If 
    \begin{align*}
        n \geq C\left(L^3 n_0^2 \left(1 + \log(n_0)\right) + L^2 \left(1 + \log\left(\frac{L}{\eta}\right)\right)\right),
    \end{align*}
    then 
    \begin{align*}
       \Prob\left( \bigcap_{\emptyset \neq u \subseteq [n_0]} \left\{\sup_{x \in [0, 1]^{n_0}} \left|D^u \cR_{\Phi^{(L)}}(x)\right| \leq C_0 |u|! \left(C_1 L\right)^{|u| - 1} \prod_{j \in u} \beta_j(\eta, n_0)\right\}\right) \geq 1 - \eta.
    \end{align*}
\end{proposition}
\begin{proof}
    By Lemma~\ref{lemma:derivative_bound_for_tanh}, there exists a constant $C_\sigma > 0$ such that  $\left|\sigma^{(r)}(x)\right| \leq A_r :=r! C_\sigma^{r - 1}$ holds for all $r \in \N$ and all $x \in \R$. We define $\overline{\kappa} := 3$ and $\overline{\alpha} := \exp(1)$. By enlarging the universal constant $C$ in the statement, we may assume that the width condition additionally implies $n\geq 2\log\left({16L}/{\eta}\right)$ and $n\geq 2\log\left({8n_0}/{\eta}\right)$. Indeed, the first inequality follows from the $L^2(1+\log(L/\eta))$ term, while the second follows from the combination of the $L^3n_0^2(1+\log n_0)$ and $L^2(1+\log(L/\eta))$ terms.

Set $\delta := {\eta}/{2}$ and $L_{\mathrm h}:=L-1$. We will apply Theorem~\ref{thm:main_high_probability_derivative_bound} to the hidden subnetwork $\Phi^{(L_{\mathrm h})} = \left(\cW^{(0)},\ldots,\cW^{(L-1)}\right)$. For this hidden subnetwork the depth is \(L_{\mathrm h}=L-1\), and $n_1=\cdots=n_{L_{\mathrm h}+1}=n$. Moreover, for every \(\ell\in[L_{\mathrm h}]\), we have $ \tau_\ell=1/{\sqrt n}$. Hence, $\tau_\ell\sqrt{n_{\ell+1}}=1$ for all \(\ell\in[L_{\mathrm h}]\).

We first verify the assumptions of
Theorem~\ref{thm:main_high_probability_derivative_bound} for the hidden
subnetwork with failure probability \(\delta\). For
\(\ell\in[L_{\mathrm h}]\), we have
\[
\begin{aligned}
    \kappa_\ell(\delta,L_{\mathrm h}) =
    \frac1{\sqrt n}
    \left(
        2\sqrt n
        +
        \sqrt{2\log\left(\frac{8L_{\mathrm h}}{\delta}\right)}
    \right) = 2 +
    \sqrt{
        \frac{
            2\log\left({16L_{\mathrm h}}/{\eta}\right)
        }{n}
    }.
\end{aligned}
\]
By increasing the constant \(C\) in the lower bound on \(n\), if necessary,
we may assume that $n \ge 2\log\left({16L_{\mathrm h}}/{\eta}\right)$. Consequently, $\kappa_\ell(\delta,L_{\mathrm h})\le3 = \overline{\kappa}$ for all \(\ell\in[L_{\mathrm h}]\). Next, for all \(1\le k\le m\le L_{\mathrm h}\),
\[
\begin{aligned}
    \left(1+\frac1{L_{\mathrm h}}\right)^{m-k+1}
    \prod_{\ell=k}^m
    \tau_\ell\sqrt{n_{\ell+1}} =
    \left(1+\frac1{L_{\mathrm h}}\right)^{m-k+1} \le
    \left(1+\frac1{L_{\mathrm h}}\right)^{L_{\mathrm h}}
    \le
    e = \overline{\alpha}.
\end{aligned}
\]
Thus the first two assumptions of
Theorem~\ref{thm:main_high_probability_derivative_bound} are satisfied. It remains to verify the width condition. We first record a simple bound for
\(\beta_j(\delta,n_0)\). Since $\tau_j^{(0)} = \sqrt{{2}/({n_0+n})}$, we have
\[
    \left(\beta_j(\delta, n_0)\right)^2
    =
    \frac{2}{n_0+n}
    \left(
        n
        +
        2\sqrt{
            n\log\left(\frac{4n_0}{\delta}\right)
        }
        +
        2\log\left(\frac{4n_0}{\delta}\right)
    \right).
\]
After increasing \(C\) again, the assumed lower bound on \(n\) implies
\[
    n
    \ge
    2\log\left(\frac{4n_0}{\delta}\right)
    =
    2\log\left(\frac{8n_0}{\eta}\right).
\]
Therefore, $\beta_j(\delta,n_0)\le3$ for all \(j\in[n_0]\). Moreover, for every \(\ell\in[L_{\mathrm h}]\),
\[
    \varepsilon_\ell
    :=
    \frac{
        \tau_\ell\sqrt{n_{\ell+1}}
    }{
        2L_{\mathrm h}\kappa_\ell(\delta,L_{\mathrm h})
    }
    =
    \frac{1}{
        2L_{\mathrm h}\kappa_\ell(\delta,L_{\mathrm h})
    }
    \ge
    \frac1{6L_{\mathrm h}}.
\]

We now estimate the entropy term $\fH_\ell \left(\varepsilon_\ell, L_{\mathrm h}, n_0, \kappa(\delta,L_{\mathrm h}), \beta(\delta,n_0) \right)$. Since \(\kappa_\ell(\delta,L_{\mathrm h})\le3\), the deterministic
exponential derivative bound gives, for all \(q\in\N\) and all
\(r\le L_{\mathrm h}\),
\[
    B_q^{(r)}(\kappa(\delta,L_{\mathrm h}))
    \le
    3^r q!
    \left(
        C_\sigma
        \sum_{p=0}^{r-1}3^p
    \right)^{q-1}
    \le
    q! C_\sigma^{q-1} 3^{rq}.
\]
Hence, for \(1\le k\le n_0\) and \(1\le\ell\le L_{\mathrm h}\),
\[
\begin{aligned}
    M_k^{(\ell)}(\kappa(\delta,L_{\mathrm h}))
    =
    A_2B_1^{(\ell-1)}B_k^{(\ell-1)}
    +
    A_1B_{k+1}^{(\ell-1)}  \le
    C
    (k+1)!
    C^k
    3^{\ell(k+1)}.
\end{aligned}
\]
Here and below, \(C\) denotes a constant depending only on \(\sigma=\tanh\),
and it may be enlarged finitely many times. Thus, $ \log\left( 1+M_k^{(\ell)}(\kappa(\delta,L_{\mathrm h})) \right) \le C \left(L n_0 + n_0\log(e n_0) \right)$ uniformly in \(1\le k\le n_0\) and \(1\le\ell\le L_{\mathrm h}\).
Using the definition of \(\fH_\ell\) and $\sum_{k=1}^{n_0}\binom{n_0}{k}\leq2^{n_0}$, together with
\(\beta_j(\delta,n_0)\le3\) and \(\varepsilon_\ell^{-1}\le6L_{\mathrm h}\), we obtain
\[
\begin{aligned}
    \fH_\ell
    \left(
        \varepsilon_\ell,
        L_{\mathrm h},
        n_0,
        \kappa(\delta,L_{\mathrm h}),
        \beta(\delta,n_0)
    \right) &\le
    \log\left(
        \sum_{k=1}^{n_0}
        \binom{n_0}{k}
        \prod_{j=1}^{n_0}
        \left(
            1+
            C L^2n_0
            M_k^{(\ell)}(\kappa(\delta,L_{\mathrm h}))
        \right)
    \right) \\
    &\le
n_0\log2
+
n_0
\max_{1\leq k\leq n_0}
\log\left(
1+CL^2n_0M_k^{(\ell)}
\right) \\
&\le
    C
    \left(
        L n_0^2
        +
        n_0^2\log(e n_0)
        +
        n_0\log(eL)
    \right) \le
    C L n_0^2\left(1+\log(n_0)\right).
\end{aligned}
\]
Therefore, after increasing the constant \(C\) in the statement if necessary,
the assumed lower bound
\[
    n
    \ge
    C
    \left(
        L^3n_0^2(1+\log n_0)
        +
        L^2\left(1+\log\left(\frac L\eta\right)\right)
    \right)
\]
implies, for all \(\ell\in[L_{\mathrm h}]\),
\[
    n
    \ge
    8L_{\mathrm h}^2
    \left(
        \fH_\ell
        \left(
            \varepsilon_\ell,
            L_{\mathrm h},
            n_0,
            \kappa(\delta,L_{\mathrm h}),
            \beta(\delta,n_0)
        \right)
        +
        \log\left(\frac{2L_{\mathrm h}}{\delta}\right)
    \right).
\]
Thus, all assumptions of Theorem~\ref{thm:main_high_probability_derivative_bound} are satisfied for
the hidden subnetwork with depth \(L_{\mathrm h}\) and failure probability \(\delta\). Hence, there exists an event \(E_{\mathrm h}\) such that $\Prob(E_{\mathrm h})\ge1-\delta=1- {\eta}/{2}$ and on \(E_{\mathrm h}\),
\[
    \left\|
        D^v\cR_{\Phi^{(L-1)}}(x)
    \right\|_2
    \le
    C_0^{\mathrm h}|v|!
    \left(C_1^{\mathrm h}L\right)^{|v|-1}
    \prod_{j\in v}\beta_j(\delta,n_0)
\]
holds simultaneously for all non-empty \(v\subseteq[n_0]\) and all \(x\in[0,1]^{n_0}\). We now estimate the scalar readout layer. Since \(n_{L+1}=1\), the matrix \(\cW^{(L)}\) is a row vector in \(\R^{1\times n}\). By Xavier initialization,
\[
    \cW^{(L)}
    =
    \sqrt{\frac{2}{n+1}}
    \left(
        X^{(L)}_{1,1},\ldots,X^{(L)}_{1,n}
    \right).
\]
By Lemma~\ref{lemma:chi_squared_sum_concentration}, applied with
\(a_1=\cdots=a_n=1\) and $x=\log\left(4/\eta\right)$, we obtain, with probability at least \(1-\eta/2\),
\[
    \sum_{i=1}^n
    \left(X^{(L)}_{1,i}\right)^2
    \le
    n
    +
    2\sqrt{n\log\left(\frac4\eta\right)}
    +
    2\log\left(\frac4\eta\right).
\]
The lower bound on \(n\), after increasing \(C\) if necessary, implies $n\ge2\log\left(4/\eta\right)$. Therefore, on an event \(E_{\mathrm{out}}\) with $\Prob(E_{\mathrm{out}})\ge1- {\eta}/{2}$, we have
\[
\begin{aligned}
    \left\|\cW^{(L)}\right\|_{2\to2}
    &=
    \sqrt{\frac{2}{n+1}}
    \left(
        \sum_{i=1}^n
        \left(X^{(L)}_{1,i}\right)^2
    \right)^{1/2} \le
    \sqrt{
        2
        \left(
            1
            +
            2\sqrt{\frac{\log(4/\eta)}{n}}
            +
            2\frac{\log(4/\eta)}{n}
        \right)
    } \le3.
\end{aligned}
\]
By the union bound, $\Prob(E_{\mathrm h}\cap E_{\mathrm{out}}) \ge 1-\eta$. We now work on the event \(E_{\mathrm h}\cap E_{\mathrm{out}}\). Fix $\emptyset\neq u\subseteq[n_0]$ and $x\in[0,1]^{n_0}$, and write \(q:=|u|\). Since the final layer is scalar, $\cR_{\Phi^{(L)}}(x) = b^{(L+1)} + \cW^{(L)} \sigma\left( \cR_{\Phi^{(L-1)}}(x) \right)$. The multivariate Faà di Bruno formula yields
\[
    D^u\cR_{\Phi^{(L)}}(x)
    =
    \cW^{(L)}
    \sum_{\pi\in\Pi(u)}
    h^{(L)}_{|\pi|}(x,\Phi^{(L-1)})
    \odot
    \bigodot_{v\in\pi}
    D^v\cR_{\Phi^{(L-1)}}(x).
\]
Taking absolute values and using
\(\|\cW^{(L)}\|_{2\to2}\le3\), the derivative bound
\(\|h^{(L)}_{|\pi|}(x,\Phi^{(L-1)})\|_\infty\le A_{|\pi|}\), and
Lemma~\ref{lemma:euclidean_norm_of_product_inequality}, we obtain
\[
\begin{aligned}
    \left|
        D^u\cR_{\Phi^{(L)}}(x)
    \right|
    &\le
    3
    \sum_{\pi\in\Pi(u)}
    A_{|\pi|}
    \prod_{v\in\pi}
    \left\|
        D^v\cR_{\Phi^{(L-1)}}(x)
    \right\|_2.
\end{aligned}
\]
Using the hidden-layer bound on \(E_{\mathrm h}\), this gives
\[
\begin{aligned}
    \left|
        D^u\cR_{\Phi^{(L)}}(x)
    \right|
    &\le
    3
    \prod_{j\in u}\beta_j(\delta,n_0)
    \sum_{\pi\in\Pi(u)}
    A_{|\pi|}
    \left(C_0^{\mathrm h}\right)^{|\pi|}
    \left(C_1^{\mathrm h}L\right)^{q-|\pi|}
    \prod_{v\in\pi}|v|!.
\end{aligned}
\]
Grouping the partitions according to \(r=|\pi|\), using $A_r\le r!C_\sigma^{r-1}$, and applying Lemma~\ref{lemma:sum_prod_of_factorial}, we get
\[
\begin{aligned}
    \left|
        D^u\cR_{\Phi^{(L)}}(x)
    \right|
    &\le
    3q!
    \prod_{j\in u}\beta_j(\delta,n_0)
    \sum_{r=1}^q
    \binom{q-1}{r-1}
    C_\sigma^{r-1}
    \left(C_0^{\mathrm h}\right)^r
    \left(C_1^{\mathrm h}L\right)^{q-r} \\
    &=
    3C_0^{\mathrm h}q!
    \left(
        C_1^{\mathrm h}L
        +
        C_\sigma C_0^{\mathrm h}
    \right)^{q-1}
    \prod_{j\in u}\beta_j(\delta,n_0).
\end{aligned}
\]
Since \(L\ge2\), we have $C_1^{\mathrm h}L+C_\sigma C_0^{\mathrm h} \le \left(C_1^{\mathrm h}+C_\sigma C_0^{\mathrm h}\right)L$. Thus,
\[
    \left|
        D^u\cR_{\Phi^{(L)}}(x)
    \right|
    \le
    3C_0^{\mathrm h}q!
    \left[
        \left(C_1^{\mathrm h}+C_\sigma C_0^{\mathrm h}\right)L
    \right]^{q-1}
    \prod_{j\in u}\beta_j(\delta,n_0).
\]
Finally, we replace \(\delta=\eta/2\) by \(\eta\) in the \(\beta\)-factor. Since \(n_0\ge1\) and \(\eta\in(0,1)\),
\[
    \log\left(\frac{8n_0}{\eta}\right)
    =
    \log\left(\frac{4n_0}{\eta}\right)+\log2
    \le
    2\log\left(\frac{4n_0}{\eta}\right).
\]
Define $f(t):=n+2\sqrt{nt}+2t$. Since $f$ is increasing and $f(2t)\leq2f(t)$ for $t\geq0$, we obtain $\beta_j(\delta,n_0)^2 \leq 2\beta_j(\eta,n_0)^2$. Hence, $\beta_j(\delta,n_0)\leq\sqrt2\,\beta_j(\eta,n_0)$ and 
\[
    \prod_{j\in u}\beta_j(\delta,n_0)
    \le
    2^{q/2}
    \prod_{j\in u}\beta_j(\eta,n_0).
\]
Since $2^{q/2}=\sqrt{2}\,(\sqrt{2})^{q-1}$, we may take, for example, $C_0:=3\sqrt{2}\,C_0^{\mathrm h}$ and $C_1:=\sqrt{2}\left(C_1^{\mathrm h} + C_\sigma C_0^{\mathrm h}\right)$.
Then
\[
    \left|D^u\cR_{\Phi^{(L)}}(x)\right|
    \leq
    C_0 q!(C_1L)^{q-1}
    \prod_{j\in u}\beta_j(\eta,n_0).
\]
This estimate holds simultaneously for all non-empty \(u\subseteq[n_0]\) and
all \(x\in[0,1]^{n_0}\) on the event
\(E_{\mathrm h}\cap E_{\mathrm{out}}\), which has probability at least
\(1-\eta\). This proves the proposition.
\end{proof}

\section{Consequences for quasi-Monte Carlo methods and Lipschitz continuity}
\label{sec:qmc}

The estimate in Proposition~\ref{prop:xavier_scalar_readout} is uniform in the input and has the product-and-order-dependent form that is commonly used in the analysis of quasi-Monte Carlo methods. In this section we first record the resulting weighted Sobolev and quadrature estimates. We then derive a high-probability bound for the Lipschitz constant and compare its dimension dependence with the estimates for random ReLU networks in \cite{dirksen_lipschitz}. Throughout the section, the network has scalar output and satisfies the assumptions of Proposition~\ref{prop:xavier_scalar_readout}.

For every non-empty $u\subseteq[n_0]$, define
\begin{equation*}
    B_u(\eta):=C_0|u|!(C_1L)^{|u|-1}\prod_{j\in u}\beta_j(\eta,n_0).
\end{equation*}
The conclusion of Proposition~\ref{prop:xavier_scalar_readout} can then be written without any ambiguity concerning the simultaneous quantifiers as
\begin{equation}\label{eq:qmc_uniform_derivative_event}
    \Prob\left(\max_{\emptyset\neq u\subseteq[n_0]}\sup_{x\in[0,1]^{n_0}}\frac{|D^u\cR_{\Phi^{(L)}}(x)|}{B_u(\eta)}\leq1\right)\geq1-\eta.
\end{equation}
All estimates below are consequences of the event in~\eqref{eq:qmc_uniform_derivative_event}. Let $s:=n_0$, and let $\gamma=(\gamma_u)_{\emptyset\neq u\subseteq[s]}$ be a family of positive weights. For a sufficiently smooth function $f:[0,1]^s\to\R$, define
\begin{equation}\label{eq:weighted_mixed_sobolev_seminorm}
    |f|_{\mathcal H_{s,\gamma}}^2:=\sum_{\emptyset\neq u\subseteq[s]}\frac{1}{\gamma_u}\int_{[0,1]^s}|D^uf(x)|^2\,dx.
\end{equation}
This is a first-order Sobolev seminorm with dominating mixed smoothness. The same argument also applies to the usual anchored and unanchored variants, since the lower-dimensional averaging operators in those norms are bounded by the corresponding uniform derivative bounds.

\begin{corollary}\label{cor:qmc_sobolev_bound}
Under the assumptions of Proposition~\ref{prop:xavier_scalar_readout}, for every family of positive weights $\gamma$ we have
\begin{equation*}
    \Prob\left(\left|\cR_{\Phi^{(L)}}\right|_{\mathcal H_{n_0,\gamma}}\leq\left(\sum_{\emptyset\neq u\subseteq[n_0]}\frac{B_u(\eta)^2}{\gamma_u}\right)^{1/2}\right)\geq1-\eta.
\end{equation*}
Moreover, the same probability event works simultaneously for every deterministic choice of the positive weights $\gamma$.
\end{corollary}
\begin{proof}
On the event in~\eqref{eq:qmc_uniform_derivative_event}, for every non-empty $u\subseteq[n_0]$,
\[
    \int_{[0,1]^{n_0}}|D^u\cR_{\Phi^{(L)}}(x)|^2\,dx\leq B_u(\eta)^2,
\]
because $[0,1]^{n_0}$ has unit volume. Substitution into~\eqref{eq:weighted_mixed_sobolev_seminorm} proves the estimate. Since the event in~\eqref{eq:qmc_uniform_derivative_event} does not depend on $\gamma$, the estimate holds there for every positive weight family.
\end{proof}
For $s=n_0$, write
\[
    I_s(f):=\int_{[0,1]^s}f(x)\,dx.
\]
Let $N\geq2$, let $\boldsymbol z\in\mathbb Z^s$ have components coprime to
$N$, and let $\boldsymbol\Delta$ be uniformly distributed on $[0,1]^s$.
The randomly shifted rank-one lattice rule is
\begin{equation*}
    Q_{N,\boldsymbol z,\boldsymbol\Delta}(f)
    :=
    \frac1N\sum_{k=0}^{N-1}
    f\!\left(\left\{\frac{k\boldsymbol z}{N}
    +\boldsymbol\Delta\right\}\right),
\end{equation*}
where braces denote the componentwise fractional part. The shift is taken
independently of the random network parameters.

\begin{corollary}[Randomly shifted lattice rule]\label{cor:qmc_shifted_lattice_bound}
Assume the hypotheses of Proposition~\ref{prop:xavier_scalar_readout}, let
$N$ be a power of a prime, and let $\lambda\in(1/2,1)$. Define
\begin{equation*}
    \varrho_1(\lambda)
    :=\frac{2\zeta(2\lambda)}{(2\pi)^{2\lambda}},
    \qquad
    c_\lambda:=2^\lambda\varrho_1(\lambda),
\end{equation*}
and choose the product-and-order-dependent weights
\begin{equation}\label{eq:qmc_nn_pod_weights}
    \gamma_u^{\mathrm{NN}}
    :=
    \left(
        \frac{B_u(\eta)^2}{c_\lambda^{|u|}}
    \right)^{1/(1+\lambda)}
    \quad(\emptyset\neq u\subseteq[s]),
    \qquad
    \gamma_\emptyset^{\mathrm{NN}}:=1.
\end{equation}
Then a generating vector $\boldsymbol z$ can be constructed by a
component-by-component algorithm such that, with probability at least
$1-\eta$ over the random network parameters,
\begin{equation}\label{eq:qmc_shifted_lattice_rms_bound}
\begin{aligned}
\left(
\mathbb E_{\boldsymbol\Delta}
\left|
I_s\!\left(\cR_{\Phi^{(L)}}\right)
-Q_{N,\boldsymbol z,\boldsymbol\Delta}
 \!\left(\cR_{\Phi^{(L)}}\right)
\right|^2
\right)^{1/2} \leq
\left(\frac{2}{N}\right)^{1/(2\lambda)}
\left(
\sum_{\emptyset\neq u\subseteq[s]}
B_u(\eta)^{\frac{2\lambda}{1+\lambda}}
c_\lambda^{\frac{|u|}{1+\lambda}}
\right)^{\frac{1+\lambda}{2\lambda}}.
\end{aligned}
\end{equation}
Thus the displayed quantity is a root-mean-square error conditional on a
network realization in the event~\eqref{eq:qmc_uniform_derivative_event};
the only expectation in~\eqref{eq:qmc_shifted_lattice_rms_bound} is over
the independent random shift.
\end{corollary}
\begin{proof}
Work on the derivative event in~\eqref{eq:qmc_uniform_derivative_event}.
Corollary~\ref{cor:qmc_sobolev_bound} gives
\begin{equation}\label{eq:qmc_h_bound_for_lattice}
    \left|\cR_{\Phi^{(L)}}\right|_{\mathcal H_{s,\gamma}}^2
    \leq
    \sum_{\emptyset\neq u\subseteq[s]}
    \frac{B_u(\eta)^2}{\gamma_u}.
\end{equation}
The nonconstant part of the standard unanchored first-order Sobolev norm is
bounded by this full $L_2$ mixed-derivative seminorm. Indeed, Jensen's
inequality yields, for every nonempty $u\subseteq[s]$,
\begin{equation*}
    \int_{[0,1]^{|u|}}
    \left|
        \int_{[0,1]^{s-|u|}}D^u f(x)\,dx_{-u}
    \right|^2dx_u
    \leq
    \int_{[0,1]^s}|D^u f(x)|^2\,dx.
\end{equation*}
The standard randomly shifted lattice-rule estimate in this non-periodic
space, together with the CBC construction for prime-power $N$, states
\cite[Section~2.1(a)]{keller_lattice_survey} that
\[
    \operatorname{r.m.s.}_{\boldsymbol\Delta} e_N^{\mathrm{wor}}
    \leq
    \left(
        \frac2N
        \sum_{\emptyset\neq u\subseteq[s]}
        \gamma_u^\lambda c_\lambda^{|u|}
    \right)^{1/(2\lambda)}.
\]
Multiplying this estimate by the square root of the right-hand side of
\eqref{eq:qmc_h_bound_for_lattice} and substituting
\eqref{eq:qmc_nn_pod_weights} makes both sums equal to
\[
    \sum_{\emptyset\neq u\subseteq[s]}
    B_u(\eta)^{\frac{2\lambda}{1+\lambda}}
    c_\lambda^{\frac{|u|}{1+\lambda}}.
\]
Thus the weights balance the Sobolev-norm and worst-case-error sums and give
\eqref{eq:qmc_shifted_lattice_rms_bound}. Since the derivative event has
probability at least $1-\eta$, the asserted conditional RMS statement
follows.
\end{proof}

The product-and-order-dependent structure can be written as
\begin{equation*}
    b_j^{\mathrm{NN}}:=C_1L\,\beta_j(\eta,n_0),
    \qquad
    B_u(\eta)=\frac{C_0}{C_1L}|u|!
    \prod_{j\in u}b_j^{\mathrm{NN}}.
\end{equation*}
The following consequence is genuinely uniform in the input dimension.

\begin{corollary}[Dimension-independent RMS rate]\label{cor:qmc_dimension_independent_rate}
Fix $\eta\in(0,1)$ and $L\geq2$. Suppose there is an infinite sequence
$(\bar b_j)_{j\geq1}$, independent of $s$, such that for every dimension
under consideration
\[
    b_j^{\mathrm{NN}}\leq\bar b_j\quad(j\leq s),
    \qquad
    \sum_{j\geq1}\bar b_j^{p^*}<\infty
\]
for some $p^*\in(0,1)$. For any $\varepsilon\in(0,1/2)$, choose
\begin{equation}\label{eq:qmc_lambda_choice}
\lambda=
\begin{cases}
\dfrac1{2-2\varepsilon},
    &p^*\in(0,2/3],\\[1ex]
\dfrac{p^*}{2-p^*},
    &p^*\in(2/3,1).
\end{cases}
\end{equation}
Then the CBC lattice rules of Corollary~\ref{cor:qmc_shifted_lattice_bound}
satisfy, with probability at least $1-\eta$ over the network parameters,
\[
    \operatorname{r.m.s.}_{\boldsymbol\Delta}\!\left|
    I_s\!\left(\cR_{\Phi^{(L)}}\right)
    -Q_{N,\boldsymbol z,\boldsymbol\Delta}
     \!\left(\cR_{\Phi^{(L)}}\right)
    \right|
    =\mathcal O(N^{-r}),
    \qquad
    r=\min\left(1-\varepsilon,\frac1{p^*}-\frac12\right).
\]
The implied constant may depend on $C_0/(C_1L)$, $p^*$,
$\varepsilon$, and the fixed majorant $(\bar b_j)_{j\geq1}$ (and hence on
$\eta$ and $L$ through the chosen majorant), but it is independent of $s$,
$N$, the shift, and the particular network realization in the derivative
event.
\end{corollary}
\begin{proof}
Set $q:=2\lambda/(1+\lambda)<1$. The choices in
\eqref{eq:qmc_lambda_choice} ensure $q\geq p^*$. Hence
$\sum_{j\geq1}\bar b_j^q<\infty$: only finitely many $\bar b_j$ exceed
one, and on the remaining indices $\bar b_j^q\leq\bar b_j^{p^*}$.
With
\[
    a_j:=c_\lambda^{1/(1+\lambda)}\bar b_j^q,
    \qquad
    D:=\frac{C_0}{C_1L},
\]
the sum in~\eqref{eq:qmc_shifted_lattice_rms_bound} is bounded, uniformly
in $s$, by
\begin{align*}
&D^q\sum_{k\geq1}(k!)^q
  \sum_{\substack{u\subset\mathbb N\\|u|=k}}
  \prod_{j\in u}a_j
\leq
D^q\sum_{k\geq1}(k!)^{q-1}
  \left(\sum_{j\geq1}a_j\right)^k
<\infty.
\end{align*}
The first inequality uses
$\sum_{|u|=k}\prod_{j\in u}a_j\leq
k!^{-1}(\sum_ja_j)^k$, and the final series converges by the ratio test
because $q<1$. Corollary~\ref{cor:qmc_shifted_lattice_bound} therefore
gives the stronger dimension-independent rate $N^{-1/(2\lambda)}$. In the
two cases in~\eqref{eq:qmc_lambda_choice}, the exponent $1/(2\lambda)$
equals respectively $1-\varepsilon$ and $1/p^*-1/2$. In either case,
$r\leq 1/(2\lambda)$, so the stronger estimate implies
$\mathcal O(N^{-r})$.
\end{proof}

We emphasize three limitations of the QMC interpretation. First, Proposition~\ref{prop:xavier_scalar_readout} concerns the random network at initialization. It does not imply that the same event remains valid after unrestricted training. A complete generalization theorem therefore requires either parameter regularization, as in \cite{kuo_regularity}, or a stability estimate showing that the training trajectory remains in a region where the derivative bounds persist. Second, isotropic Xavier initialization makes the quantities $\beta_j(\eta,n_0)$ comparable for all $j$. Hence, the summability assumptions needed for dimension-independent QMC constants are not uniform as $n_0\to\infty$. Such estimates require an anisotropic initialization or regularization that enforces coordinate decay. Third, the present theorem controls only square-free mixed derivatives. This is sufficient for first-order spaces of dominating mixed smoothness, but not for higher-order periodic Korobov spaces, which require repeated derivatives and a periodic architecture.

We next turn to the first-order consequence of~\eqref{eq:qmc_uniform_derivative_event}. For $p\in[1,\infty]$, let $p'$ be its H\"older conjugate and define
\begin{equation*}
    \operatorname{Lip}_{\ell^p}(f;[0,1]^{n_0}):=\sup_{\substack{x,y\in[0,1]^{n_0}\\x\neq y}}\frac{|f(x)-f(y)|}{\|x-y\|_{\ell^p}}.
\end{equation*}

\begin{corollary}[High-probability Lipschitz estimate]\label{cor:qmc_lipschitz_bound}
Under the assumptions of Proposition~\ref{prop:xavier_scalar_readout}, for every $p\in[1,\infty]$,
\begin{equation}\label{eq:qmc_lipschitz_probability_bound}
    \Prob\left(\sup_{x\in[0,1]^{n_0}}\left\|\nabla\cR_{\Phi^{(L)}}(x)\right\|_{\ell^{p'}}\leq C_0\left\|\beta(\eta,n_0)\right\|_{\ell^{p'}}\right)\geq1-\eta.
\end{equation}
On the same event,
\begin{equation}\label{eq:qmc_lipschitz_consequence}
    \operatorname{Lip}_{\ell^p}\left(\cR_{\Phi^{(L)}};[0,1]^{n_0}\right)\leq C_0\left\|\beta(\eta,n_0)\right\|_{\ell^{p'}}.
\end{equation}
In particular, after increasing the universal constant in the width condition if necessary, $\beta_j(\eta,n_0)\leq3$ for all $j\in[n_0]$, and hence
\begin{equation*}
    \Prob\left(\sup_{x\in[0,1]^{n_0}}\left\|\nabla\cR_{\Phi^{(L)}}(x)\right\|_{\ell^{p'}}\leq3C_0n_0^{1-1/p}\right)\geq1-\eta.
\end{equation*}
Consequently, with probability at least $1-\eta$, we have $\operatorname{Lip}_{\ell^p}\left(\cR_{\Phi^{(L)}};[0,1]^{n_0}\right)\leq3C_0n_0^{1-1/p}$.
\end{corollary}
\begin{proof}
Taking $u=\{j\}$ in~\eqref{eq:qmc_uniform_derivative_event} gives $\sup_{x\in[0,1]^{n_0}}|\partial_j\cR_{\Phi^{(L)}}(x)|\leq C_0\beta_j(\eta,n_0)$ for every $j\in[n_0]$ on an event of probability at least $1-\eta$. Taking the $\ell^{p'}$-norm over the coordinates proves~\eqref{eq:qmc_lipschitz_probability_bound}. Since the cube is convex, the fundamental theorem of calculus along the line segment from $x$ to $y$, followed by H\"older's inequality, gives
\[
    |\cR_{\Phi^{(L)}}(x)-\cR_{\Phi^{(L)}}(y)|\leq\sup_{z\in[0,1]^{n_0}}\|\nabla\cR_{\Phi^{(L)}}(z)\|_{\ell^{p'}}\|x-y\|_{\ell^p}.
\]
This proves~\eqref{eq:qmc_lipschitz_consequence}. The last claim follows from $\|\beta(\eta,n_0)\|_{\ell^{p'}}\leq3n_0^{1/p'}=3n_0^{1-1/p}$.
\end{proof}

It is instructive to compare Corollary~\ref{cor:qmc_lipschitz_bound} with the results of Dirksen, Finke, Geuchen, St\"oger, and Voigtlaender \cite{dirksen_lipschitz}. They consider random ReLU networks $\Psi:\R^{n_0}\to\R$ with $L$ hidden layers of width $n$ and a variant of He initialization. However, we need to be careful if we compare those results with ours. Their estimates are global estimates on $\R^d$, the activation functions are different, and the output normalizations are not the same: the final-layer weights in \cite{dirksen_lipschitz} are standard Gaussian, while our scalar Xavier readout has variance $2/(n+1)$. Moreover, our result is only an upper bound, whereas \cite{dirksen_lipschitz} also establishes lower bounds.

\section*{Acknowledgements}

Funding was received from the European Research Council (ERC) under the European Union’s Horizon 2020 research and innovation programme (Grant agreement No. 101125225). Funding was also received from the OeAD under the Marietta Blau-Grant.

\section*{Declaration on generative AI assistance}
Generative AI tools were used for research assistance. In particular, AI assistance contributed to the formulation of mathematical ideas used in the paper, including Proposition~\ref{prop:explicit_bound_derivative} and Proposition~\ref{prop:measurable_tangent_net}. AI tools were also used for literature search and for drafting and revising parts of the exposition. All definitions, theorem statements, proofs, calculations, and references included in the manuscript were independently checked, revised, and approved by the authors. The authors assume responsibility for all content.

\appendix
\section{Auxiliary results}

The following form of the multivariate Fa\`a di Bruno formula follows from \cite[Propositions 1 and 2]{Hardy_2006}.
\begin{theorem}\label{prop:faa_di_bruno}
    Suppose $\sigma: \R \to \R$ and $g: \R^s \to \R$ are two smooth functions. If $\nu \in \N_0^s$, then
    \begin{align*}
        \partial^{\nu}\left(\sigma \circ g\right)(x) = \sum_{\pi \in \Pi\left([|\nu|]\right)} \sigma^{(|\pi|)}\left(g(x)\right)\prod_{v \in \pi} \partial^{\mu(\nu, v)}g(x),
    \end{align*}
    holds for all $x \in \R^s$, where $\mu(\nu, v)$ is defined in \eqref{eq:def_of_mu_nu_v}.
\end{theorem}
If $u=\{j_1<\cdots<j_q\}\subseteq[n_0]$ and
$\nu=\mathbf 1_u$, then the order-preserving bijection
$r\mapsto j_r$ from $[q]$ onto $u$ induces a bijection between
$\Pi([q])$ and $\Pi(u)$. Under this identification,
$\partial^{\mu(\nu,v)}=D^v$. Hence the Fa\`a di Bruno formula takes the
particularly simple square-free form
\[
    D^u(\sigma\circ g)(x)
    =
    \sum_{\pi\in\Pi(u)}
    \sigma^{(|\pi|)}(g(x))
    \prod_{v\in\pi}D^vg(x).
\]

\begin{lemma}\label{lemma:euclidean_norm_of_product_inequality}
    Let $k, n \in \N$. If $a^{(i)} \in \R^n$ for all $i \in [k]$ then
    \begin{align*}
        \left\| \bigodot_{i = 1}^k a^{(i)} \right\|_2 \leq \prod_{i = 1}^{k}\left\| a^{(i)} \right\|_2
    \end{align*}
    holds.
\end{lemma}
\begin{proof}
    We have
    \begin{align*}
        \left\| \bigodot_{i = 1}^k a^{(i)} \right\|_2^2 &= \sum_{j = 1}^n \prod_{i = 1}^k \left|a^{(i)}_j\right|^2 \leq \sum_{j = 1}^n \left(\left|a^{(1)}_j\right|^2\prod_{i = 2}^k \left\|a^{(i)}\right\|_2^2\right) = \left(\prod_{i = 2}^k \left\|a^{(i)}\right\|_2^2\right)\sum_{j = 1}^n \left|a^{(1)}_j\right|^2 = \prod_{i = 1}^k \left\|a^{(i)}\right\|_2^2.
    \end{align*}
    Taking square roots gives the claim. 
\end{proof}

\begin{lemma}\label{lemma:derivative_bound_for_tanh}
    For all $r \in \N$ we have $\left|\tanh^{(r)}(x)\right| \leq r! \left({16}/{\pi^2}\right)^{r - 1}$ for all $x\in\R$.
\end{lemma}
\begin{proof}
    For every $x \in \R$ we define the closed disk $D(x) := \left\{z \in \C: |x - z| \leq \pi/4\right\}$ and we have
    \begin{align*}
        D(x) \subseteq M := \left\{z \in \C: |\im(z)| < \frac{\pi}{2} \right\}
    \end{align*}
    If $a, b \in \R$ and $a + ib = z \in M$ then
    \begin{align*}
        |\sinh(z)|^2 &= \left(\sinh(a)\cos(b)\right)^2 + \left(\cosh(a) \sin(b)\right)^2 = \left(\sinh(a)\right)^2\left(1 - \left(\sin(b)\right)^2\right) + \left(\cosh(a) \sin(b)\right)^2 \\
        &= \left(\sinh(a)\right)^2 + \left(\sin(b)\right)^2 \left(\left(\cosh(a)\right)^2 - \left(\sinh(a)\right)^2\right) = \left(\sinh(a)\right)^2 + \left(\sin(b)\right)^2
    \end{align*}
    and
    \begin{align*}
        |\cosh(z)|^2 &= \left(\cosh(a)\cos(b)\right)^2 + \left(\sinh(a) \sin(b)\right)^2 = \left(\cosh(a)\right)^2\left(1 - \left(\sin(b)\right)^2\right) + \left(\sinh(a) \sin(b)\right)^2 \\
        &= \left(\cosh(a)\right)^2 - \left(\sin(b)\right)^2\left(\left(\cosh(a)\right)^2 - \left(\sinh(a)\right)^2\right) = \left(\cosh(a)\right)^2 - \left(\sin(b)\right)^2.
    \end{align*}
    If $|b| \leq \pi/4$ then $\left(\sin(b)\right)^2 \leq 1/2$ and 
    \begin{align*}
        |\sinh(z)|^2 &= \left(\sinh(a)\right)^2 + \left(\sin(b)\right)^2 \leq \left(\sinh(a)\right)^2 + 1 - \left(\sin(b)\right)^2 \\
        &=  \left(\cosh(a)\right)^2 - \left(\sin(b)\right)^2 = |\cosh(z)|^2.
    \end{align*}
    In this case we obtain $\left|\tanh(z)\right|^2 = {|\sinh(z)|^2}/{|\cosh(z)|^2} \leq 1$. By Cauchy's estimate we have $\left|\tanh^{(r)}(x)\right| \leq r! \left({4}/{\pi}\right)^r$ for all $x \in \R$ and all $r \in \N$. Since the holomorphic extension agrees with the real-valued function on the real axis, its complex derivatives restricted to $\R$ coincide with the corresponding real derivatives. For $r = 1$ we have $\left|\tanh^\prime(x)\right| \leq 1 = 1! \left({16}/{\pi^2}\right)^0$. For $r \geq 2$ we have $2r - 2 \geq r$ and since $4/\pi > 1$ we obtain
    \begin{align*}
        \left|\tanh^{(r)}(x)\right| \leq r! \left(\frac{4}{\pi}\right)^r \leq r! \left(\frac{4}{\pi}\right)^{2(r - 1)} = r! \left(\frac{16}{\pi^2}\right)^{r - 1}.
    \end{align*}
\end{proof}

\begin{lemma}\label{lemma:sum_prod_of_factorial}
    For all $r \in \N$ and all finite sets $u \neq \emptyset$ we have
    \begin{align*}
        \sum_{\pi \in \Pi_r(u)} \prod_{v \in \pi} |v|! = \frac{|u|!}{r!} \binom{|u| - 1}{r - 1}.
    \end{align*}
\end{lemma}
\begin{proof}
We prove the identity by strong induction on $m:=|u|$, simultaneously for
every $r\in\N$. If $m=1$, then for $r=1$ both sides equal one, whereas
for $r>1$ both sides vanish.

Suppose that $n\geq1$ and that the identity holds for every nonempty
finite set of size at most $n$ and every block number in $\N$. Let $u$ be
a finite set with $|u|=n+1$, and fix $r\in\N$. If $r>n+1$, then
$\Pi_r(u)=\emptyset$ and $\binom{n}{r-1}=0$, so both sides vanish. If
$r=1$, then $\Pi_1(u)=\{\{u\}\}$, and both sides equal $(n+1)!$.
It remains to consider $2\leq r\leq n+1$.

Fix $x\in u$ and set $v:=u\setminus\{x\}$. Define
\[
    A:=\{\pi\in\Pi_r(u):\{x\}\in\pi\},
    \qquad
    B:=\{\pi\in\Pi_r(u):\{x\}\notin\pi\}.
\]
Then $A$ and $B$ are disjoint and $\Pi_r(u)=A\cup B$. Define
$f:A\to\Pi_{r-1}(v)$ by
$f(\pi):=\pi\setminus\{\{x\}\}$. This is a bijection with inverse
$f^{-1}(\rho):=\rho\cup\{\{x\}\}$. The induction hypothesis, applied
to the set $v$ and the block number $r-1$, gives
\begin{equation*}
    \sum_{\pi\in A}\prod_{w\in\pi}|w|!
    =\frac{n!}{(r-1)!}\binom{n-1}{r-2}.
\end{equation*}

Next, define
\[
    C:=\{(\rho,y):\rho\in\Pi_r(v),\ y\in\rho\}.
\]
For $\pi\in B$, let $h(\pi):=w$, where $w$ is the unique block of
$\pi$ containing $x$, and define $g:B\to C$ by
\[
    g(\pi):=
    \left(
        \bigl(\pi\setminus\{h(\pi)\}\bigr)
        \cup\{h(\pi)\setminus\{x\}\},
        h(\pi)\setminus\{x\}
    \right).
\]
The map $g$ is a bijection with inverse
\[
    g^{-1}(\rho,y):=
    \bigl(\rho\setminus\{y\}\bigr)\cup\{y\cup\{x\}\}.
\]
The induction hypothesis, now applied to the set $v$ and the block number
$r$, yields
\begin{align*}
    \sum_{\pi\in B}\prod_{w\in\pi}|w|! =\sum_{\rho\in\Pi_r(v)}\sum_{y\in\rho}
      (|y|+1)!\prod_{\substack{w\in\rho\\w\neq y}}|w|!= \sum_{\rho\in\Pi_r(v)}
      \left(\prod_{w\in\rho}|w|!\right)
      \sum_{y\in\rho}(|y|+1)  =(n+r)\frac{n!}{r!}\binom{n-1}{r-1}.
\end{align*}
Combining the preceding identities, we obtain
\begin{align*}
    \sum_{\pi\in\Pi_r(u)}\prod_{w\in\pi}|w|! =\frac{n!}{(r-1)!}\binom{n-1}{r-2}
      +(n+r)\frac{n!}{r!}\binom{n-1}{r-1} =\frac{(n+1)!}{r!}\binom{n}{r-1}
      =\frac{|u|!}{r!}\binom{|u|-1}{r-1}.
\end{align*}
This proves the identity for every $r\in\N$ and completes the strong
induction.
\end{proof}

\begin{lemma}\label{lemma:tangent_set_closed}
For fixed $n_0,L,\beta,\alpha,\gamma$, the set
$\Theta_L^{\tangent}(n_0,\beta,\alpha,\gamma)$ is closed in $\Theta_L$.
In particular, if $\Phi:\Omega\to\Theta_L$ is measurable, then
$\{\Phi\in\Theta_L^{\tangent}(n_0,\beta,\alpha,\gamma)\}$ is an event.
\end{lemma}
\begin{proof}
For $\ell\in[L]$ and $\emptyset\neq u\subseteq[n_0]$, define
\[
F_{\ell,u}(\varphi,x)
:=
\left\|W^{(\ell)}
G_u^{(\ell)}(x,\varphi^{(\ell-1)})\right\|_2
-\alpha_\ell
\left\|G_u^{(\ell)}(x,\varphi^{(\ell-1)})\right\|_2
-\gamma_\ell\prod_{j\in u}\beta_j.
\]
The map $(\varphi,x)\mapsto F_{\ell,u}(\varphi,x)$ is continuous.
Since $[0,1]^{n_0}$ is compact, the map
\[
\varphi\longmapsto
\max_{x\in[0,1]^{n_0}}F_{\ell,u}(\varphi,x)
\]
is continuous. Hence,
\[
\left\{\varphi: \max_{x\in[0,1]^{n_0}}F_{\ell,u}(\varphi,x)\le0\right\}
\]
is closed. Taking the finite intersection over $\ell$ and non-empty
$u\subseteq[n_0]$ proves the claim.
\end{proof}

\begin{proposition}\label{prop:linear_recurrence_relation}
Let $N\in\N$, let $a_0,\ldots,a_N\in\R$, and let $f_0,\ldots,f_{N-1},g_0,\ldots,g_{N-1}\in[0,\infty)$. If $a_{n+1}\leq f_na_n+g_n$ holds for every $n=0,\ldots,N-1$, then
\[
    a_n
    \leq
    \left(\prod_{p=0}^{n-1}f_p\right)a_0
    +
    \sum_{m=0}^{n-1}
    \left(\prod_{p=m+1}^{n-1}f_p\right)g_m
\]
for every $n=0,\ldots,N$.
\end{proposition}
\begin{proof}
The assertion follows by induction over $n$. The case $n=0$ is immediate. If the estimate holds for $n$, then multiplication by $f_n\geq0$ and the recurrence inequality give
\[
\begin{aligned}
    a_{n+1} \leq
    f_na_n+g_n  \leq
    \left(\prod_{p=0}^{n}f_p\right)a_0
    +
    \sum_{m=0}^{n-1}
    \left(\prod_{p=m+1}^{n}f_p\right)g_m
    +g_n,
\end{aligned}
\]
which is the desired formula for $n+1$, using the convention for empty products.
\end{proof}

We use the following singular-value estimate from \cite[Corollary~5.35]{Vershynin_2012}.
\begin{proposition}\label{prop:singular_values_gaussian_matrix}
    If $A$ is an $m \times n$ random matrix whose entries $A_{i, j}$ are independent standard normal random variables, then for every $t \geq 0$ we have
    \begin{align*}
        \sqrt{m} - \sqrt{n} - t \leq s_{\min}(A) \leq s_{\max}(A) \leq \sqrt{m} + \sqrt{n} + t
    \end{align*}
    with probability at least $1 - 2\exp(-t^2/2)$. Here $s_{\min}(A)$ and $s_{\max}(A)$ denote respectively the smallest and the largest singular value of $A$. 

    This means in particular that the matrix norm
    \begin{align*}
        \|A\|_{2 \to 2} \leq \sqrt{m} + \sqrt{n} + t
    \end{align*}
    is satisfied with probability at least $1 - 2\exp(-t^2/2)$.
\end{proposition}

We also use the following concentration inequality from \cite[Lemma~1]{laurent_massart}.
\begin{lemma}\label{lemma:chi_squared_sum_concentration}
    Let $n \in \N$ and $X_1, \ldots, X_n$ be i.i.d. standard normal random variables. If $a \in \R^n$ is a vector with non-negative entries, then 
    \begin{align*}
        \Prob\left(\sum_{i = 1}^n a_i\left(X_i^2 - 1\right) \geq 2\|a\|_2 \sqrt{x} + 2\|a\|_\infty x\right) \leq \exp(-x)
    \end{align*}
    holds for all $x > 0$.
\end{lemma}

By \cite[Lemma~26.8]{munkres_topology} we have
\begin{lemma}\label{lemma:tube_lemma}
    Let $X$ and $Y$ be topological spaces with $Y$ compact, and consider the product space $X \times Y$. If $x \in X$ and $U$ is an open set in the product topology containing $\{x\} \times Y$, then there exists an open subset $V$ of $X$ such that $\{x\} \times Y \subseteq V \times Y \subseteq U \subseteq X \times Y$.
\end{lemma}
As a consequence of the tube Lemma~\ref{lemma:tube_lemma} we get the following result.
\begin{corollary}\label{cor:projection_is_closed}
    If $X$ and $Y$ are topological spaces and $Y$ is compact, then the projection $p: X \times Y \to X$ given by $p(x, y) = x$ is a closed map.
\end{corollary}
\begin{proof}
    Let $A$ be any closed set in $X \times Y$ and define $U := (X \times Y) \setminus A$. Suppose $x \in X \setminus p(A)$. This means that for all $y \in Y$ we have $(x, y) \in U$ and consequently, $\{x\} \times Y \subseteq U$. By the Tube Lemma~\ref{lemma:tube_lemma} there exists an open subset $V$ of $X$ such that $\{x\} \times Y \subseteq V \times Y \subseteq U \subseteq X \times Y$. Since $\{x\} \subseteq V$ and $V \cap p(A) = \emptyset$ we conclude that $p(A)$ is a closed set in $X$.
\end{proof}

\begin{lemma}\label{lemma:measurable_lexicographic_representative}
    Let \(S\) be a topological space, let \(d\in\N\), and let
    \(C\subseteq[0,1]^d\) be non-empty and compact. Suppose that
    \(\Gamma\subseteq S\times C\) is closed. For \(s\in S\), define $\Gamma_s := \{x\in C:(s,x)\in\Gamma\}$ and $D := \{s\in S:\Gamma_s\neq\emptyset\}$. Then \(D\) is closed. Moreover, every \(\Gamma_s\), \(s\in D\), has a unique lexicographically smallest element, and the map $\lambda:D\to C$ given by $\lambda(s):=\lexmin\Gamma_s$ is Borel measurable.
\end{lemma}
\begin{proof}
    Since \(C\) is compact and \(\Gamma\) is closed, the projection $D=\proj_S(\Gamma)$ is closed by Corollary~\ref{cor:projection_is_closed}.

    For \(n\in\N_0\), let
    \[
        \mathcal I_n
        :=
        \left\{
            \left[
                \frac{r-1}{2^n},
                \frac{r}{2^n}
            \right]
            :
            r=1,\ldots,2^n
        \right\}.
    \]
    If
    \[
        I=
        \left[
            \frac{r-1}{2^n},
            \frac{r}{2^n}
        \right]
        \in\mathcal I_n,
    \]
    define its left and right children by
    \[
        I^{\mathrm L}
        :=
        \left[
            \frac{r-1}{2^n},
            \frac{2r-1}{2^{n+1}}
        \right] \quad \text{and} \quad 
        I^{\mathrm R}
        :=
        \left[
            \frac{2r-1}{2^{n+1}},
            \frac{r}{2^n}
        \right].
    \]

    We first record a measurability observation. Suppose that \(1\leq j\leq d\), and that $J_i:D\to\mathcal I_{m_i}$, for $i=1,\ldots,j$ are Borel measurable maps. Then the set
    \[
    \begin{aligned}
        E(J_1,\ldots,J_j)
        :=
        \big\{
            s\in D:
            \Gamma_s
            \cap
            \big(
                J_1(s)\times\cdots\times J_j(s)
                \times[0,1]^{d-j}
            \big)
            \neq\emptyset
        \big\}
    \end{aligned}
    \]
    is Borel measurable. Indeed, each \(J_i\) has finite range. Hence,
    \[
    \begin{aligned}
        E(J_1,\ldots,J_j)
        =
        \bigcup_{A_1,\ldots,A_j}
        \Bigg[
            \left(\bigcap_{i=1}^j J_i^{-1}(\{A_i\}) \right) \cap \Big\{
                s\in D:
                \Gamma_s
                \cap
                \big(
                    A_1\times\cdots\times A_j
                    \times[0,1]^{d-j}
                \big)
                \neq\emptyset
            \Big\}
        \Bigg],
    \end{aligned}
    \]
    where $A_i \in \cI_{m_i}$ and the union is finite. For fixed $A_1,\ldots, A_j$, set $K:= C\cap \left( A_1\times\cdots\times A_j\times[0,1]^{d-j} \right)$. Then $K$ is compact, and the corresponding hit set is $\proj_D(\Gamma \cap \left(D \times K\right))$. Hence, it is closed by Corollary~\ref{cor:projection_is_closed}. Therefore, \(E(J_1,\ldots,J_j)\) is Borel.

    We now construct, successively for \(j=1,\ldots,d\), Borel measurable
    interval-valued maps $I_{j,n}:D\to\mathcal I_n$ for $n\in\N_0$, such that, for every \(s\in D\), we have $I_{j,n+1}(s)\subseteq I_{j,n}(s)$ and $\operatorname{diam}(I_{j,n}(s))=2^{-n}$.

    Suppose first that the interval maps \(I_{i,n}\) have already been
    constructed for every \(i<j\) and every \(n\in\N_0\), and define $a_i(s)$ as the unique element of $\bigcap_{n=0}^\infty I_{i,n}(s)$ for all $i < j$. For \(j=1\), there are no preceding coordinates.

    Set $I_{j,0}(s):=[0,1]$ for every \(s\in D\). Suppose \(I_{j,n}\) has been constructed.
    Let $L_{j,n}(s):=I_{j,n}(s)^{\mathrm L}$ and $R_{j,n}(s):=I_{j,n}(s)^{\mathrm R}$. Both maps have finite range and are Borel measurable. Define
    \[
    \begin{aligned}
        H_{j,n}
        :=
        \Big\{
            s\in D:
            \exists x\in\Gamma_s
            \text{ such that }
            x_i=a_i(s)\text{ for }i<j
            \text{ and }
            x_j\in L_{j,n}(s)
        \Big\}.
    \end{aligned}
    \]
    We claim that \(H_{j,n}\) is Borel. In fact,
    \[
        H_{j,n}
        =
        \bigcap_{N=0}^{\infty}
        E\left(
            I_{1,N},
            \ldots,
            I_{j-1,N},
            L_{j,n}
        \right),
    \]
    where for \(j=1\) the preceding interval maps are omitted. To verify this equality, the inclusion from left to right is immediate. Conversely, suppose \(s\) belongs to every set on the right. Then, for every \(N\), the set
    \[
    \begin{aligned}
        K_N
        :=
        \Gamma_s
        \cap
        \Big(
            I_{1,N}(s)
            \times\cdots\times
            I_{j-1,N}(s)
            \times
            L_{j,n}(s)
            \times
            [0,1]^{d-j}
        \Big)
    \end{aligned}
    \]
    is non-empty and compact. The sequence \((K_N)_{N\geq0}\) is
    decreasing. Hence compactness gives $\bigcap_{N=0}^{\infty}K_N\neq\emptyset$. Every point in this intersection satisfies $x_i=a_i(s)$ for $i < j$ and \(x_j\in L_{j,n}(s)\). Thus, \(s\in H_{j,n}\), proving the claim.

    We now define
    \[
        I_{j,n+1}(s)
        :=
        \begin{cases}
            L_{j,n}(s), & s\in H_{j,n},\\
            R_{j,n}(s), & s\notin H_{j,n}.
        \end{cases}
    \]
    Since \(H_{j,n}\) is Borel and the two child maps are Borel,
    \(I_{j,n+1}\) is Borel measurable. For fixed \(s\in D\), define $K_s^{j-1} := \left\{x \in \Gamma_s: x_i=a_i(s) \text{ for }i<j \right\}$. Inductively, this is a non-empty compact set. Hence the continuous
    coordinate projection \(x\mapsto x_j\) attains its minimum on
    \(K_s^{j-1}\). Denote that minimum by $a_j(s) := \min\{x_j:x\in K_s^{j-1}\}$.

    We claim that $a_j(s)\in I_{j,n}(s)$ holds for every \(n\). This is clear for \(n=0\). Suppose it holds for
    \(n\). If the left child contains a point of \(K_s^{j-1}\), then,
    since \(a_j(s)\) is the minimum of the \(j\)-th coordinates on
    \(K_s^{j-1}\), it also belongs to the left child. If the left child
    contains no such point, then \(a_j(s)\) belongs to the right child.
    Thus the claim follows by induction.

    Since the intervals are nested and have lengths tending to zero, $\bigcap_{n=0}^{\infty}I_{j,n}(s) = \{a_j(s)\}$. If \(u_{j,n}(s)\) denotes the right endpoint of \(I_{j,n}(s)\), then
    \(u_{j,n}\) is Borel measurable and $a_j(s)=\lim_{n\to\infty}u_{j,n}(s)$. Therefore, \(a_j:D\to[0,1]\) is Borel measurable.

    This completes the recursive construction for every
    \(j=1,\ldots,d\). Define $\lambda(s) := \big(a_1(s),\ldots,a_d(s)\big)$. Since each coordinate \(a_j\) is Borel measurable, the map
    \(\lambda:D\to[0,1]^d\) is Borel measurable.

    Finally, define recursively $K_s^0:=\Gamma_s$ and $ K_s^j := \left\{x\in K_s^{j-1}:x_j=a_j(s) \right\}$. Each \(K_s^j\) is non-empty and compact. The set \(K_s^d\) contains
    exactly the point \(\lambda(s)\), so $\lambda(s)\in\Gamma_s\subseteq C$. By construction, \(a_1(s)\) is the smallest first coordinate in \(\Gamma_s\); among points having that first coordinate, \(a_2(s)\) is the smallest second coordinate; and so on.
    Consequently, $\lambda(s)=\lexmin\Gamma_s$. The lexicographically smallest element is necessarily unique.
\end{proof}

\begin{lemma}\label{lemma:partial_derivatives_imply_lipschitz}
    Let $s, d \in \N$ and $F \in C^1\left([0, 1]^s; \R^d\right)$. If there exist constants $c_1, \ldots, c_s \geq 0$ such that $\left\|\partial_m F(z)\right\|_2 \leq c_m$ holds for all $z \in [0, 1]^s$ and all $m \in [s]$, then 
    \begin{align*}
        \left\|F(x) - F(y)\right\|_2 \leq \sum_{m = 1}^s c_m\left|x _m - y_m \right|
    \end{align*}
    holds for all $x, y \in [0, 1]^s$.
\end{lemma}
\begin{proof}
    By the fundamental theorem of calculus we have
    \begin{align*}
        \left\|F(y) - F(x)\right\|_2 &= \left\|\int_0^1 DF\left(x + t(y - x)\right)(y - x) dt\right\|_2 \\
        &\leq \int_0^1 \left\|\sum_{m = 1}^s(y_m - x_m)\partial_mF(x + t(y - x))\right\|_2 dt \\
        &\leq \int_0^1 \left(\sum_{m = 1}^s \left|y_m-x_m\right| \left\|\partial_mF(x + t(y - x))\right\|_2\right) dt \leq \sum_{m = 1}^s c_m \left|y_m - x_m\right|.
    \end{align*}
\end{proof}

\begin{lemma}\label{lemma:lipschitz_bound_for_tangent_maps}
    Let $\ell \in \N$, and let $u\subseteq[n_0]$ be nonempty. If $v \in \Theta_{\ell - 1}^{\mat}(n_0, \beta, \kappa)$, then
    \begin{align*}
        \left\|G^{(\ell)}_u(x, v) - G^{(\ell)}_u(y, v)\right\|_2 \leq M^{(\ell)}_{|u|}(\kappa) \prod_{j \in u} \beta_j \sum_{m = 1}^{n_0} \beta_m \left|x_m - y_m\right|
    \end{align*}
    holds for all $x, y \in [0, 1]^{n_0}$, where $M^{(\ell)}_k(\kappa) := A_2 B_1^{(\ell-1)}(\kappa)B_k^{(\ell-1)}(\kappa) + A_1B^{(\ell - 1)}_{k + 1}(\kappa)$.
\end{lemma}
\begin{proof}
    For $m \in [n_0]$ we have
    \begin{align*}
        \partial_m G^{(\ell)}_u(x,v) = h^{(\ell)}_2(x,v) \odot \partial_m \cR_v(x)
        \odot D^u \cR_v(x)
        + h^{(\ell)}_1(x,v) \odot \partial_m D^u\cR_v(x).
    \end{align*}
    Since $v \in \Theta_{\ell - 1}^{\mat}(n_0, \beta, \kappa)$,
    Theorem~\ref{thm:derivative_bound_from_matrix_norm_bound} and
    Lemma~\ref{lemma:euclidean_norm_of_product_inequality} give
    \begin{align*}
        \left\|\partial_m G^{(\ell)}_u(x,v)\right\|_2
        &\leq A_2 \left\|\partial_m \cR_v(x)\right\|_2
        \left\|D^u \cR_v(x)\right\|_2
        + A_1 \left\|\partial_m D^u\cR_v(x)\right\|_2 \\
        &\leq A_2 B^{(\ell - 1)}_1(\kappa) \beta_m B^{(\ell - 1)}_{|u|}(\kappa)\prod_{j \in u} \beta_j + A_1 B^{(\ell - 1)}_{|u| + 1}(\kappa) \beta_m \prod_{j \in u}\beta_j = M^{(\ell)}_{|u|}(\kappa) \beta_m \prod_{j \in u}\beta_j.
    \end{align*}
    The asserted Lipschitz estimate follows from
    Lemma~\ref{lemma:partial_derivatives_imply_lipschitz}.
\end{proof}

\begin{lemma}\label{lemma:stability_of_normalization}
    If $\left(E, \left\|\cdot\right\|\right)$ is a normed vector space, then 
    \begin{align*}
        \left\|\frac{a}{\|a\|} - \frac{b}{\|b\|}\right\| \leq \frac{2\|b - a\|}{\min\{\|a\|, \|b\|\}}
    \end{align*}
    holds for all $a, b \in E\setminus\{0\}$.
\end{lemma}
\begin{proof}
    By the triangle inequality we have
    \begin{align*}
        \left\|\frac{a}{\|a\|} - \frac{b}{\|b\|}\right\| = \left\|\frac{a - b}{\|a\|} + b\left(\frac{1}{\|a\|} - \frac{1}{\|b\|}\right) \right\| \leq \frac{\left\|a - b\right\|}{\|a\|} + \left\|b\right\|\left|\frac{1}{\|a\|} - \frac{1}{\|b\|}\right| =  \frac{\left\|a - b\right\|}{\|a\|} + \frac{\left|\left\|b\right\| - \left\|a\right\|\right|}{\|a\|}. 
    \end{align*}
    By the reverse triangle inequality we have $\left|\left\|b\right\| - \left\|a\right\|\right| \leq \left\|b - a\right\|$ and consequently
    \begin{align*}
         \left\|\frac{a}{\|a\|} - \frac{b}{\|b\|}\right\| \leq \frac{\left\|a - b\right\|}{\|a\|} + \frac{\left|\left\|b\right\| - \left\|a\right\|\right|}{\|a\|} \leq \frac{2\left\|a - b\right\|}{\|a\|} \leq \frac{2\|b - a\|}{\min\{\|a\|, \|b\|\}}.
    \end{align*}
\end{proof}

The following conditional-expectation identity is standard; see \cite[Chapter~4]{durrett2019probability}.
\begin{lemma}\label{lemma:independent_cond_exp}
    Let $X$ and $Y$ be independent random variables taking values in
    Euclidean spaces, and let $f$ be a bounded Borel-measurable function
    on the corresponding product space. Define
    $g(x):=\E[f(x,Y)]$. Then
    \[
        \E\bigl[f(X,Y)\mid\sigma(X)\bigr]=g(X)
    \]
    almost surely.
\end{lemma}

\printbibliography

\end{document}